\documentclass[runningheads]{llncs}
\usepackage{graphicx}
\usepackage[dvipsnames]{xcolor}

\usepackage{tikzit}

\usepackage{marginnote}
\usepackage{ifthen,changepage}
\usepackage{xargs}

\newcommandx{\leftmarginnote}[2][2=0pt]
           {\checkoddpage
            \ifoddpage
              {\reversemarginpar\marginnote{#1}[#2]}
            \else
              {\marginnote{#1}[#2]}
            \fi}
\newcommandx{\rightmarginnote}[2][2=0pt]
           {\checkoddpage
            \ifoddpage
              {\marginnote{#1}[#2]}
            \else
              {\reversemarginpar\marginnote{#1}[#2]}
            \fi}

\newcommand{\nvcol}[1]{\textcolor{Black}{#1}}
\newcommand{\smcol}[1]{\textcolor{NavyBlue}{#1}}
\newcommand{\mbcol}[1]{\textcolor{Black}{#1}}
\newcommand{\nvt}[1]{\nvcol{#1}}

\newcommand{\mbi}[1]{\mbcol{[#1 ---~MB]}}
\newcommand{\smi}[1]{\smcol{[#1 ---~SM]}}
\newcommand{\nvi}[1]{\nvcol{[#1 ---~NV]}}
\newcommand{\nvr}[2][0pt]{{\rightmarginnote{\small\nvi{#2}}[#1]}}
\newcommand{\mbr}[2][0pt]{{\rightmarginnote{\small\mbi{#2}}[#1]}}
\newcommand{\smr}[2][0pt]{{\rightmarginnote{\small\smi{#2}}[#1]}}
\newcommand{\nvl}[2][0pt]{{\leftmarginnote{\small\nvi{#2}}[#1]}}

\newcommand{\flagtext}[1]{\textcolor{red}{\textbf{#1}}}
\newcommand{\tocite}[1]{\flagtext{[#1]}}
\newcommand{\todo}[1]{\flagtext{(to do: #1)}}

\newcommand{\tocheck}{\flagtext{(to do: check this)}}

\usepackage{amsmath}
\usepackage{amssymb}
\usepackage{mathrsfs} 
\usepackage{mathtools}

\usepackage[bookmarks,bookmarksopen,bookmarksdepth=2]{hyperref}

\hypersetup{
  colorlinks   = true,    
  urlcolor     = black,    
  linkcolor    = black,    
  citecolor    = black      
}

\usepackage{cleveref}

\usepackage{stmaryrd} 

\newcommand{\comp}{\mathbin{\fatsemi}}

\newcommand{\mg}{\mkern4mu\rule[-2.5pt]{1.25pt}{10pt}\mkern4mu}

\usepackage{bbm}
\newcommand{\one}{\mathbbm{1}}

\newenvironment{nalign}{
    \begin{equation}
    \begin{aligned}
}{
    \end{aligned}
    \end{equation}
    \ignorespacesafterend
}

\newcommand{\id}{\mathsf{id}} 
\newcommand{\del}{\mathsf{del}} 
\newcommand{\cpy}{\mathsf{copy}} 

\makeatletter
\newcommand{\mathoverlap}[2]{\mathpalette\mathoverlap@{{#1}{#2}}}
\newcommand{\mathoverlap@}[2]{\mathoverlap@@{#1}#2}
\newcommand{\mathoverlap@@}[3]{\ooalign{$\m@th#1#2$\crcr\hidewidth$\m@th#1#3$\hidewidth}}
\makeatother
\newcommand{\klcirc}{\bullet} %
\newcommand*{\smallklcirc}{\raisebox{0.18ex}{\scalebox{0.66}{$\klcirc$}}}

\newcommand{\klto}{\mathoverlap{\rightarrow}{\smallklcirc\,}}
\newcommand{\xklto}[1]{\mathoverlap{\xrightarrow{#1}}{\smallklcirc\,}}

\newcommand{\todot}{\mathrel{\klto}}

\newcommand{\sprime}{^{\smash{\mathrlap{\prime}}}}

\usepackage{tikz}
\usepackage{ifthen}
\usetikzlibrary{shapes}
\usetikzlibrary{shapes.misc}

\newcommand{\MinLength}[3]{%
\ifthenelse{\lengthtest{\the#1<\the#2}}%
           {\setlength{#3}{#1}}%
           {\setlength{#3}{#2}}%
}

\newcommand{\MaxLength}[3]{%
\ifthenelse{\lengthtest{\the#1>\the#2}}%
           {\setlength{#3}{#1}}%
           {\setlength{#3}{#2}}%
}

\newlength{\xscale}
\newlength{\yscale}
\newlength{\mscale}

\newcommand{\setstringscale}[2] {%
	\setlength{\xscale}{#1}%
	\setlength{\yscale}{#2}%
	\MinLength{\xscale}{\yscale}{\mscale}%
}

\newcommand{\stringdisplaystyle} {%
	\setstringscale{0.33cm}{0.23cm}%
}
\newcommand{\stringinlinestyle} {%
	\setstringscale{0.2cm}{0.12cm}%
}

\newlength{\boxscale}
\makeatletter 
\newcommand{\autostringstyle}{%
    \if@display
	\stringdisplaystyle
    \else
	\stringinlinestyle
    \fi
}
\makeatother 

\newcommand{\wire}[1][1]{\draw[rounded corners=#1\mscale]}

\newcommand{\fg}[1]{\begin{pgfonlayer}{foreground}#1\end{pgfonlayer}}

\newcommand{\blackdot}[1]{\fg{\draw[fill=black] (#1) circle (0.25\mscale);}}

\newcommand{\sqbox}[3]{
\fg{
	\draw (#1) node[
		fill=white,
		rounded rectangle,
		draw,
		minimum height=#3\boxscale,
		inner sep = 1pt,
		rounded rectangle left arc=none,
		rounded rectangle right arc=none,
		rounded rectangle arc length = 180,
	] {\small #2};}
}

\newcommand{\stoch}[3]{
\fg{
	\draw (#1) node[
		fill=white,
		rounded rectangle,
		draw,
		minimum height=#3\boxscale,
		inner sep = 1pt,
		rounded rectangle left arc=none,
		rounded rectangle arc length = 120
	] {\small #2};
}
}

\newcommand{\llabel}[2]{\fg{\draw (#1) +(-3.5pt,0) node[inner sep = 0] {\scriptsize $#2$};}}
\newcommand{\rlabel}[2]{\fg{\draw (#1) +(3.5pt,0) node[inner sep = 0] {\scriptsize $#2$};}}

\newcommand{\ulabel}[2]{\fg{\draw (#1) +(0,5pt) node[inner sep = 0] {\scriptsize $#2$};}}
\newcommand{\dlabel}[2]{\fg{\draw (#1) +(0,-5pt) node[inner sep = 0] {\scriptsize $#2$};}}

\newcommand{\uulabel}[2]{\fg{\draw (#1) +(0,6pt) node[inner sep = 0] {\scriptsize $#2$};}}
\newcommand{\ddlabel}[2]{\fg{\draw (#1) +(0,-6pt) node[inner sep = 0] {\scriptsize $#2$};}}

\newcommand{\stringdiagram}[1]{%
\mbox{
\autostringstyle%
\begin{tikzpicture}[x=\xscale, y=\yscale, baseline={([yshift=-0.45ex]current bounding box.center)}]
\pgfdeclarelayer{grid layer}
\pgfdeclarelayer{foreground}
\pgfsetlayers{grid layer,main,foreground}
#1
\end{tikzpicture}%
}}

\newcommand{\pbox}[2]{
	\left.
	#1
	\right._{\raisebox{0.5\yscale}{$#2$}}
}

\newcommand{\inlinekernel}[3]{$\!\!\!\stringdiagram{
\wire (0,1)--(4,1);
\stoch{2,1}{$#2$}{2}
\llabel{0,1}{#1};
\rlabel{4,1}{#3};
}$}

\newcommand{\inlinemachine}[3]{$\!\!\!\stringdiagram{
\wire (0,1)--(4,1);
\wire[0.5] (0,3)--(0.5,3)--(1,2)--(2,2);
\stoch{2,1.5}{$#3$}{2}
\llabel{0,3}{ #2};
\llabel{0,1}{ #1};
\rlabel{4,1}{ #1};
}$}

\newcommand{\ysgamma}{\inlinemachine{Y}{S}{\gamma}}

\tikzstyle{function}=[fill=white, draw=black, shape=rectangle, minimum width=0.75cm, minimum height=0.75cm]
\tikzstyle{kernel}=[fill=white, draw=black, shape=rounded rectangle, rounded rectangle west arc=0pt, minimum height=0.6cm, minimum width=0.6cm, tikzit fill={rgb,255: red,240; green,240; blue,240}]
\tikzstyle{copy}=[fill=black, draw=black, shape=circle, minimum width=0.1cm, minimum height=0.1cm, inner sep=0pt]
\tikzstyle{concat}=[fill=white, draw=black, shape=circle, minimum width=0.1cm, minimum height=0.1cm]
\tikzstyle{simple}=[fill=white, draw=black, shape=circle, minimum height=0.5cm, minimum width=0.5cm]
\tikzstyle{edge-label}=[fill=white, draw=none, shape=circle, minimum height=0.1cm, minimum width=0.1cm, inner sep=0]

\tikzstyle{directed-edge}=[->]
\tikzstyle{thick}=[->, line width=1pt, tikzit draw={rgb,255: red,191; green,0; blue,64}]

\begin{document}
\title{Interpreting \nvcol{Dynamical} Systems as Bayesian Reasoners}
%
%
\author{Nathaniel Virgo\inst{1}\orcidID{0000-0001-8598-590X} \and
Martin Biehl\inst{2}\orcidID{0000-0002-1670-6855} \and
Simon McGregor\inst{3}}
%
\authorrunning{N.\ Virgo et al.}
%

\institute{$^{1}$Earth-Life Science Institute, Tokyo Institute of Technology, Tokyo 152-8550, Japan \\
$^{2}$Araya Inc., Tokyo 107-6024, Japan \\
$^{3}$University of Sussex, Falmer, UK
}

%
\maketitle              
%


\begin{abstract}
A central concept in active inference is that the internal states of a physical
system parametrise probability measures over states of the external world. These
can be seen as an agent's beliefs, expressed as a Bayesian prior \nvt{or posterior}. Here we begin
the development of a general theory that would tell us when it is appropriate to
interpret states as representing beliefs in this way. We focus on the case 
in which a system can be interpreted as performing either
Bayesian filtering or Bayesian inference. We provide formal definitions of what
it means for such an interpretation to exist, using techniques from category theory.

\keywords{Bayesian filtering  \and Bayesian Inference \and Category Theory.}
\end{abstract}
%
%
%


\section{Introduction}

A question of current interest is \emph{what does it mean for a physical system to be an agent?}
That is, given a physical system that interacts with an environment, when does it make sense to say that the system is \emph{learning about its environment} or \emph{trying to achieve a goal}, rather than merely being dynamically coupled to its environment?
Here we confine ourselves to the first of these, in a simple form:
given a physical system that is influenced by its surroundings, under what circumstances can it be said to be performing inference, such that its internal states could be said to contain `knowledge' or `beliefs' about the outside world?

Our approach has something in common with Dennet's intentional stance \cite{dennett_true_1981}, in that on the one hand we treat the question of whether a system is performing inference as a matter of interpretation, but on the other hand we draw a strong connection between interpretations and the underlying physical dynamics.
We provide a formal notion of interpretation for the particular cases we are interested in (Bayesian filtering and Bayesian inference), such that the question of \emph{whether a system can be consistently interpreted in a particular way} is mathematically well-defined and has a definite answer.

The question of how to identify agents in physical systems has been addressed in several ways.
Some works focus on whether a system's actions can be seen as pursuing a goal \cite{orseau_agents_2018,kolchinsky_semantic_2018}, with \cite{mcgregor_bayesian_2017} taking an explicitly Dennettian approach.
Others focus more on the question of identifying which part of a system should be identified as the agent \cite{biehl_towards_2016,beer_autopoiesis_2004,beer_cognitive_2014,krakauer_information_2020,albantakis_macro_2020}, or on understanding what the external world looks like from the agent's point of view \cite{ay_umwelt_2015,beer_autopoiesis_2004,beer_cognitive_2014}.
%
Another approach, which we take here, is to regard the system's internal state as \emph{parametrising} its beliefs.
That is, there is a function mapping the system's physical state to a probability measure that can be seen as a Bayesian prior.
This is a key component of work on the Free Energy Principle (FEP) \cite{friston_free_2019,da_costa_bayesian_2021,parr_markov_2020} and also of \cite{still_thermodynamics_2012}, although our approach differs from these in that our model is not derived from the dynamics of the true environment. 
The notion that agency is closely related to parametrisation is also central to recent approaches to agency based on category theory \cite{capucci_towards_2021,st_clere_smithe_cyber_2021,capucci_translating_2021}.

The idea that states of a system parametrise Bayesian probability distributions appears more broadly in the Bayesian brain literature \cite{knill_bayesian_2004,ma_neural_2014} and has also arisen in cell biology \cite{libby_noisy_2007,nakamura_connection_2021}.
%
%
Our contribution is to make the concept much more formal, and in so doing, to shed light on the precise relationship between the interpretation level and the underlying physical level.

On a technical level we formulate the problem of Bayesian filtering at an abstract level using the tools of category theory.
%
\nvt{This part of the work is inspired by}
 \cite{jacobs_channel-based_2020}, which formulates the notion of conjugate prior \nvt{in terms of category theory} in a similar way.
\nvt{Conjugate priors are convenient because they ensure the functional form of the posterior is the same as the posterior, in Bayesian belief updating.
At the same time they can be seen as a special case of Bayesian filtering, as we explain in \cref{consistency-for-bayesian-interpretations} and \cref{unpacking-filtering}.}
%
\nvt{Formulating filtering in this way allows us to clearly} distinguish the role of the physical machine from the more semantic level at which we can talk about priors and posteriors.
%
\nvt{We then}
flip this perspective around, asking, for a given system, whether it can be interpreted as implementing Bayesian filtering, and if so under which model.
In this respect our approach generalises that of \cite{biehl_dynamics_2020}, who studied the special case of the Dirichlet distribution (which is conjugate prior to a categorical distribution) in the context of interpreting a physical system as performing inference.

%
%
One thing our approach makes clear is that a given system may have more than one interpretation, and the ``correct'' interpretation cannot be determined from the system's dynamics alone.
Another important aspect of our framework is that an interpretation only depends on the system's internal dynamics, and not on the dynamics of the external world.
Because of this, a system's presumed beliefs might not match the true dynamics of the world at all --- its beliefs might be consistent but incorrect --- and indeed we can construct examples where the world ``as the system sees it'' has a different causal structure from the world as it really is. (Compare \cref{bayesian-network-context} to \cref{bayesian-network-iid}.)

\section{Definitions and Results}

\subsection{Technical preliminaries}
\label{preliminaries}

In the following, we use the concepts of \emph{measurable space} and \emph{Markov kernel}.
By measurable space we mean a set equipped with a $\sigma$-algebra, i.e.\ the kind of thing on which a probability measure can be defined.
An example of a measurable space is a finite set, and a reader who is only interested in the finite case could mentally substitute ``finite set'' wherever we say ``measurable space'' and ``probability distribution'' wherever we say ``probability measure.''

Given a measurable space $X$, we write $P(X)$ for the set of all probability measures over $X$.
Given measurable spaces $X$ and $Y$, a Markov kernel is a function $\kappa\colon X\to P(Y)$ that maps elements of $X$ to probability measures over $Y$, with an additional technical requirement that the function $\kappa$ be measurable.
Markov kernels are closely related to conditional probability, but they are different in that a Markov kernel defines a probability measure over $Y$ for every element $X$, regardless of whether any probability distribution has been defined over $X$ or what form such a distribution has.
In the case where $Y$ is a finite set, we write $\kappa(y\mg x)$ for the probability that the kernel $\kappa$ assigns to $y$ when given the input~$x$.

We also make use of a graphical notation known as \emph{string diagrams}, which comes from the literature on category-theoretic probability \cite{cho_disintegration_2019,fritz_synthetic_2020}.
This notation provides a convenient way to reason about how Markov kernels relate to one another.
The full technical description of this calculus can be found in \cite{fritz_synthetic_2020} or \cite{cho_disintegration_2019}, but we provide a brief intuitive explanation in \cref{string-diagram-intro}, aimed at readers with no category theory background, along with references to further reading.

\subsection{Machines and interpretations}

We are concerned with interpreting a physical system as performing inferences of some kind on its inputs.
We therefore begin by defining a notion corresponding to a physical system that can take an input from the outside world, which leads to a change in state, which might be stochastic and might depend on the input.

\begin{definition}
 A \emph{machine} consists of two measurable spaces, $Y$ (the state space) and $S$ (the input space), together with a Markov kernel $\gamma\colon Y\times S \to P(Y)$ called the update kernel.
\end{definition}

The idea is that the machine is in reality only a part of some larger stochastic process.
We might typically think of this broader context as represented by the following causal Bayesian network, although this is not the only possibility.
\begin{equation}
\label{bayesian-network-context}
\raisebox{-0.5\height}{\includegraphics[height=5.5em]{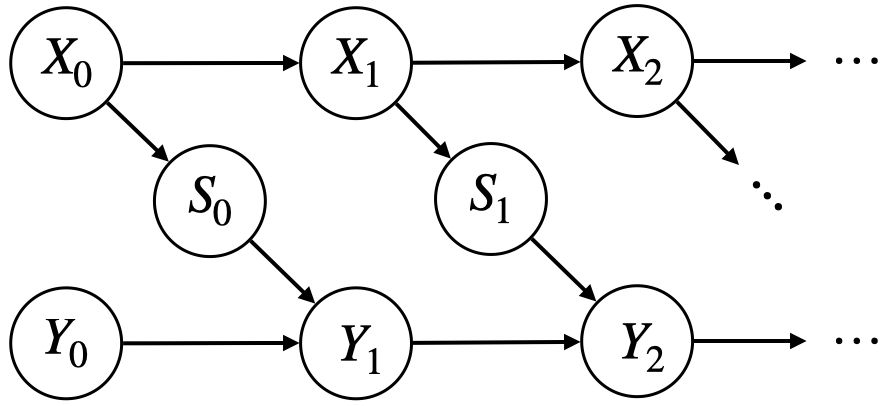}}
\end{equation}
Here the variables $X_0,X_1,\dots$ represent the states of the external world at different times, which are hidden from the machine's perspective. $S_0,S_1,\dots$ are the observable ``sensor values'' that the machine can access, all of which have the same sample space, given by $S$. Similarly, $Y_0,Y_1,\dots$ are the machine's internal state at each time step and each has the sample space $Y$.
We assume that each of the nodes $Y_1,Y_2,\dots$ in this network are associated with the same kernel, $\gamma$.

The nodes $X_0$ and $Y_0$ represent distributions over initial states of the agent and the world.
A common ancestor of these nodes could be added, to represent the possibility that the initial states are correlated.
%
\nvt{There is no need for any stationarity assumption in our framework; the initial distribution in \cref{bayesian-network-context} can be arbitrary.}

Even though we might think of the machine as existing in a context along these lines, our notion of interpretation does not depend on the machine's external environment at all, but \emph{only} on the machine's internal dynamics.
That is, on the measurable spaces $Y$ and $S$ and on the update kernel $\gamma$.
This is because, informally speaking, a reasoner may have consistent \emph{but wrong} beliefs about the external world, and we wish to include this possibility in our framework. 
(In particular, a system that reasons correctly in one environment might be placed in a different environment where the same inferences are no longer correct.)
Our notion of interpretation will include a notion of ``beliefs'' about the external world.
These must be consistent with the machine's internal dynamics and the inputs it receives, but they need not relate in any particular way to any ground truth about the process by which those inputs are generated, since we regard that as something the reasoner has no direct access to.

Because of this we will rarely reason about causal models directly, and instead will express our definitions directly in terms of Markov kernels and the relationships between them.
%
%
%
The string diagram notation, explained \cref{string-diagram-intro}, will be indispensable for this.



We now describe our central concept: an \emph{interpretation} of a machine.
Here we will present only two kinds of interpretations, \emph{Bayesian filtering interpretations} and an important special case, \emph{Bayesian inference interpretations}, but we expect these to fit naturally into a much broader family of concepts.

The most important component of an interpretation is what we term an \emph{interpretation map}, a function that maps the physical state of a machine to something that we can think of as a belief about some external world.
In the cases we are concerned with in this paper, a ``belief'' will be a probability measure over some hidden variable.
%
%
An interpretation of a machine will be an interpretation map together with some additional data (the \emph{model} as defined below), such that a \emph{consistency equation} is obeyed.

For a given machine there may be many possible interpretations. 
We use the term \emph{reasoner} for a machine together with a particular choice of interpretation.

In the case of {Bayesian inference interpretations} and {Bayesian filtering interpretations}, in which `beliefs' are probability measures over a hidden variable, the interpretation map is a Markov kernel $\psi_H\colon Y\to P(H)$.
Instead of interpreting this stochastically, we think of it as a function that takes a state $y\in Y$ and returns a probability measure $\psi_H(y)$ (the belief) over $H$.
This is to be thought of as the reasoner's subjective knowledge (a Bayesian prior or posterior) about the hypotheses in $H$.
%
%
This kernel plays quite a different role from those associated with the graphical model in \cref{bayesian-network-context}, since its purpose is to map states to beliefs, rather than to model causal influences between random variables.

For an interpretation to be {consistent}, the reasoner's beliefs must update in the appropriate way when the machine receives new data.
The precise meaning of this will depend on what kind of interpretation we are using.
For the interpretations we describe here it is given by Bayes' rule, in a form that we state precisely below.
In future work we can imagine interpretations based on other principles, such as approximate Bayes (e.g.\ via the free energy principle).

The idea is that a machine by itself is merely a (possibly stochastic) dynamical process, but if a consistent interpretation exists then it is at least consistent to ascribe a \emph{meaning} to its states.
%

It should be noted that, for a given machine, the question of \emph{whether it can be consistently interpreted in a particular way} is in principle an empirical one, since it depends on the machine's update kernel, which can in principle be measured.
However, in general a given machine might have multiple non-equivalent consistent interpretations, and one cannot distinguish between these empirically by looking only at the system's internal dynamics.%
\footnote{We leave open the possibility that they could be distinguished by looking at some broader context, e.g.\ by discovering that a device's designer intended a particular interpretation, or that evolution selected for a particular interpretation.}
Consequently the relationship between interpretations and the empirical, physical world is rather subtle, and one should keep in mind that our notion of ``consistent reasoner'' unavoidably involves an element of choice in which interpretation to adopt.

\subsection{Consistency for Bayesian Interpretations}
\label{consistency-for-bayesian-interpretations}

We begin with the more general case of Bayesian filtering interpretations.
The idea is that a reasoner has a \emph{model} of the environment's dynamics.
This model is part of the interpretation, and need not match the true environment dynamics.

In the case of filtering, such a model can be described as a Markov kernel $\!\stringdiagram{
\wire (0,-1)--(-4,-1);
\wire[0.5] (0,-3)--(-0.5,-3)--(-1,-2)--(-2,-2);
\stoch{-2,-1.5}{$\kappa$}{2}
\rlabel{0,-3}{ S};
\rlabel{0,-1}{ H};
\llabel{-4,-1}{ H};
}$,
that is, $\kappa\colon H\to P(H\times S)$.
The idea is that $H$ is the space of possible {hidden states} of the external world, as modelled by the reasoner.
The kernel $\kappa$ models a step of this hypothetical external world's evolution, during which it both changes to a new state in $H$ and also emits an observable sensor value in~$S$.

A Bayesian filtering interpretation of a machine $\!\!\!\stringdiagram{
\wire (0,1)--(4,1);
\wire[0.5] (0,3)--(0.5,3)--(1,2)--(2,2);
\stoch{2,1.5}{$\gamma$}{2}
\llabel{0,3}{ S};
\llabel{0,1}{ Y};
\rlabel{4,1}{ Y};
}$ will then consist of a choice of interpretation map $\!\!\stringdiagram{
\wire (0,1)--(4,1);
\stoch{2,1}{$\psi_H\!$}{2}
\llabel{0,1}{Y};
\rlabel{4,1}{H};
}$ as described above, together with a choice of model $\!\stringdiagram{
\wire (0,-1)--(-4,-1);
\wire[0.5] (0,-3)--(-0.5,-3)--(-1,-2)--(-2,-2);
\stoch{-2,-1.5}{$\kappa$}{2}
\rlabel{0,-3}{ S};
\rlabel{0,-1}{ H};
\llabel{-4,-1}{ H};
}$.
The kernel $\kappa$ thus describes a reasoner's beliefs about the next hidden state and the next sensor value, given the current hidden state.

%
%

Given the kernels $\psi_H$ and $\kappa$ we can define another kernel, which we also consider to be an interpretation map,
\begin{equation}
\label{filtering-big-interpretation-map}
\pbox{
  \stringdiagram{
 \llabel{0,0}{Y};
 \wire(0,0) -- (2,0);
 \stoch{2.5,0}{$\psi_{S,H\sprime,H}$}{6};
 \wire(2,1.5) -- (5,1.5);
 \wire(2,0) -- (5,0);
 \wire(2,-1.5) -- (5,-1.5);
 \rlabel{5,1.5}{H};
 \rlabel{5,0}{H};
 \rlabel{5,-1.5}{S};
 }
 \,\coloneqq
   \stringdiagram{
 \ulabel{0.3,0}{H};
 \wire(-4,0) -- (2,0);
 \llabel{-4,0}{Y};
 \stoch{2,0}{$\kappa$}{4};
 \stoch{-2,0}{$\psi_H$}{3};
 \blackdot{-0.5,0};
 \wire[3] (-0.5,0) -- (-0.5,3) -- (4,3);
 \wire(2,1) -- (4,1);
 \wire(2,-1) -- (4,-1);
 \rlabel{4,3}{H};
 \rlabel{4,1}{H};
 \rlabel{4,-1}{S};
 }
 }{.}
\end{equation}
The kernel $\psi_{H,H',S}$ maps a state of the machine $y\in Y$ to a joint distribution over $S$ and two copies of $H$, which we think of as the reasoner's beliefs about the the next sensor value, the next hidden state, and the current hidden state.
%
%
We also define its marginals,
\begin{equation}
\label{filtering-interpretation-SH'}
\pbox{
\stringdiagram{
 \llabel{0,0}{Y};
 \wire(0,0) -- (2,0);
 \stoch{2.5,0}{$\psi_{S,H'}$}{4};
 \wire(2,0.7) -- (4.7,0.7);
 \wire(2,-0.7) -- (4.7,-0.7);
 \rlabel{4.7,0.7}{H};
 \rlabel{4.7,-0.7}{S};
 }
 \,\,\coloneqq
 \stringdiagram{
 \llabel{0,0}{Y};
 \wire(0,0) -- (2,0);
 \stoch{2.5,0}{$\psi_{S,H\sprime,H}$}{4};
 \wire(2,1) -- (4.8,1);
 \blackdot{4.8,1}
 \wire(2,0) -- (5,0);
 \wire(2,-1) -- (5,-1);
 \rlabel{5,0}{H};
 \rlabel{5,-1}{S};
 }
 \,\,=
 \stringdiagram{
 \ulabel{0.3,0}{H};
 \wire(-3,0) -- (2,0);
 \llabel{-3,0}{Y};
 \stoch{2,0}{$\kappa$}{4};
 \stoch{-1.2,0}{$\psi_H$}{3};
 \wire(2,1) -- (4,1);
 \wire(2,-1) -- (4,-1);
 \rlabel{4,1}{H};
 \rlabel{4,-1}{S};
 }
 }{}
\end{equation}
and
\begin{equation}
\label{filtering-interpretation-S}
 \pbox{
 \stringdiagram{
 \llabel{0.5,0}{Y};
 \wire(0.5,0) -- (4.7,0);
 \stoch{2.5,0}{$\psi_{S}$}{3};
 \rlabel{4.7,0}{S};
 }
 \,\,\coloneqq
 \stringdiagram{
 \llabel{0,0}{Y};
 \wire(0,0) -- (2,0);
 \stoch{2.5,0}{$\psi_{S,H\sprime,H}$}{4};
 \wire(2,1) -- (4.8,1);
 \blackdot{4.8,1}
 \wire(2,0) -- (4.8,0);
 \blackdot{4.8,0}
 \wire(2,-1) -- (5,-1);
 \rlabel{5,-1}{S};
 }
 \,\,=
 \stringdiagram{
 \ulabel{0.3,0}{H};
 \wire(-3,0) -- (2,0);
 \llabel{-3,0}{Y};
 \stoch{2,0}{$\kappa$}{4};
 \stoch{-1.2,0}{$\psi_H$}{3};
 \wire(2,1) -- (3.75,1);
 \blackdot{3.75,1};
 \wire(2,-1) -- (4,-1);
 \rlabel{4,-1}{S};
 }
}{,}
\end{equation}
which represent the reasoner's beliefs about the next hidden state and the next input, and about only the next input, respectively.
We can now state the consistency requirement for a Bayesian filtering interpretation.

\begin{definition}
\label{bayesian-filtering-interpretation-definition}
 Given a machine $\!\!\!\stringdiagram{
\wire (0,1)--(4,1);
\wire[0.5] (0,3)--(0.5,3)--(1,2)--(2,2);
\stoch{2,1.5}{$\gamma$}{2}
\llabel{0,3}{ S};
\llabel{0,1}{ Y};
\rlabel{4,1}{ Y};
},$ a \emph{consistent Bayesian filtering interpretation} of $\gamma$ is given by a measurable space $H$ together with Markov kernels $\!\!\stringdiagram{
\wire (0,1)--(4,1);
\stoch{2,1}{$\psi_H\!$}{2}
\llabel{0,1}{Y};
\rlabel{4,1}{H};
}$ and $\!\!\stringdiagram{
\wire (0,1)--(2,1);
\wire (2,0)--(4,0);
\wire (2,2)--(4,2);
\stoch{2,1}{$\kappa$}{2.5}
\llabel{0,1}{H};
\rlabel{4,2}{H};
\rlabel{4,0}{S};
}$ that satisfy
\begin{equation}
\label{bayesian-filtering-consistency}
\stringdiagram {
 \wire (0,0) -- (11.5,0);
 \ulabel{0,0}{Y};
 \ulabel{11.5,0}{Y};
 \blackdot{1,0};
 \wire[3] (1,0) -- (1,3) -- (5,3);
 \wire (5,2) -- (11.5,2);
 \ulabel{11.5,2}{S};
 \stoch{4.5,3}{$\psi_{S,H'}$}{4};
 \stoch{9.5,0.25}{$\gamma$}{3};
 \blackdot{7.5,2};
 \ulabel{8,2}{S};
 \wire[1.5] (7.5,2) -- (7.5,0.5) -- (9.5,0.5);
 \wire[2] (5,4) -- (11.5,4);
 \ulabel{11.5,4}{H};
}
\,\, = 
\pbox{
\stringdiagram {
 \wire (2,0) -- (15,0);
 \ulabel{2,0}{Y};
 \ulabel{15,0}{Y};
 \blackdot{3,0};
 \wire[2] (3,0) -- (3,2) -- (6,2);
 \stoch{5.25,2}{$\psi_{S}$}{3};
 \wire (6,2) -- (15,2);
 \ulabel{15,2}{S};
 \stoch{9,0.25}{$\gamma$}{3};
 \blackdot{7,2};
 \ulabel{7.4,2}{S};
 \wire[1.5] (7,2) -- (7,0.5) -- (9,0.5);
 \wire[4] (10.5,0) -- (10.5,4) -- (15,4);
 \blackdot{10.5,0};
 \ulabel{11.3,0}{Y};
 \stoch{13.2,4}{$\psi_{H}$}{3};
 \ulabel{15,4}{H};
}
}{,}
\end{equation}
with $\psi_{S,H'}$ and $\psi_S$ given by \cref{filtering-interpretation-SH',filtering-interpretation-S}.
\end{definition}

The left-hand-side of \cref{bayesian-filtering-consistency} can be read as sampling from the reasoner's joint beliefs about the next hidden state and the next input, and then feeding the corresponding value of $S$ into the machine as an input.
The right-hand-side can be read as sampling from the reasoner's belief about its next input, feeding the result in as its next input, and then sampling from its resulting (posterior) belief about what it now sees as the current hidden state.
The equation says that these two procedures must give the same result.

In \cref{unpacking-filtering} we give some further intuition for this definition, by considering the case where $S$, $Y$ and $H$ are finite sets.
An important consequence is that, in the finite case, whether a given interpretation is consistent or not only depends on which states are reachable from which other states under a given input; the actual transition probabilities are irrelevant beyond that.
We expect an analogous statement to hold more generally.

Another important consequence discussed in \cref{unpacking-filtering} is that there is a large class of machines that only admit trivial interpretations.
At least in the finite case, non-trivial interpretations can only exist if some transitions are impossible, in the sense that there is a zero probability of transitioning from $y\in Y$ to $y'\in Y$ under the input $s\in S$.

In \cref{bayesian-filtering-proof} we give a more technical proof, using string diagrams, that at least in the case of deterministic machines, a machine with a consistent Bayesian filtering interpretation can indeed be regarded as performing a Bayesian filtering task.
This can be seen as extending some of the ideas in \cite{jacobs_channel-based_2018} on conjugate priors to the more general case of Bayesian filtering.

An important special case of \cref{bayesian-filtering-interpretation-definition} is where the model $\kappa$ is such that the hidden state does not change over time.
In this case $H$ can be thought of as an unknown parameter of a statistical model, with the sensor inputs being independent and identically distributed samples from the model.
We call these \emph{Bayesian inference interpretations}.
%
%
%
Although the following follows from \cref{bayesian-filtering-interpretation-definition} under this assumption, we write it as a separate definition.

\begin{definition}
\label{bayesian-inference-interpretation-definition}
 Given a machine $\!\!\stringdiagram{
\wire (0,1)--(4,1);
\wire[0.5] (0,3)--(0.5,3)--(1,2)--(2,2);
\stoch{2,1.5}{$\gamma$}{2}
\llabel{0,3}{ S};
\llabel{0,1}{ Y};
\rlabel{4,1}{ Y};
},$ a \emph{consistent Bayesian inference interpretation} of $\gamma$ is given by a measurable space $H$ together with Markov kernels $\!\!\stringdiagram{
\wire (0,1)--(4,1);
\stoch{2,1}{$\psi_H\!$}{2}
\llabel{0,1}{Y};
\rlabel{4,1}{H};
}$ and $\!\!\stringdiagram{
\wire (0,1)--(4,1);
\stoch{2,1}{$\phi$}{2}
\llabel{0,1}{H};
\rlabel{4,1}{S};
}$ that satisfy
\begin{equation}
\label{bayesian-inference-consistency}
\stringdiagram {
 \wire (0,0) -- (11.5,0);
 \ulabel{0,0}{Y};
 \ulabel{11.5,0}{Y};
 \blackdot{1,0};
 \wire[2] (1,0) -- (1,2) -- (11.5,2);
 \ulabel{11.5,2}{S};
 \stoch{3,2}{$\psi_H\!$}{3};
 \stoch{6,2}{$\phi$}{3};
 \blackdot{4.5,2};
 \ddlabel{4.5,2}{H};
 \stoch{9.5,0.25}{$\gamma$}{3};
 \blackdot{7.5,2};
 \ulabel{8,2}{S};
 \wire[1.5] (7.5,2) -- (7.5,0.5) -- (9.5,0.5);
 \wire[2] (4.5,2) -- (4.5,4) -- (11.5,4);
 \ulabel{11.5,4}{H};
}
\, = \!
\pbox{
\stringdiagram {
 \wire (0,0) -- (15,0);
 \ulabel{0,0}{Y};
 \ulabel{15,0}{Y};
 \blackdot{1,0};
 \wire[2] (1,0) -- (1,2) -- (15,2);
 \ulabel{15,2}{S};
 \stoch{3,2}{$\psi_H\!$}{3};
 \ulabel{4.25,2}{H};
 \stoch{5.5,2}{$\phi$}{3};
 \stoch{9,0.25}{$\gamma$}{3};
 \blackdot{7,2};
 \ulabel{7.5,2}{S};
 \wire[1.5] (7,2) -- (7,0.5) -- (9,0.5);
 \wire[4] (10.5,0) -- (10.5,4) -- (15,4);
 \blackdot{10.5,0};
 \ulabel{11.3,0}{Y};
 \stoch{13.2,4}{$\psi_H\!$}{3};
 \ulabel{15,4}{H};
}
}{.}
\end{equation}
\end{definition}

In \cref{conjugate-priors} we show that this definition is closely related to the notion of a conjugate prior, and in particular to the definition of conjugate prior given by \cite{jacobs_channel-based_2018} in terms of Markov kernels and string diagrams.
Finally, in \cref{unpacking-inference} we unpack \cref{bayesian-inference-consistency} in more familiar terms, showing that in the discrete case it does indeed correspond to Bayes' theorem as usually understood.

In a Bayesian inference interpretation we interpret the {reasoner} as 
\emph{assuming} its inputs are i.i.d.\ samples from some distribution, but this need not mean that they actually are.
Under a consistent Bayesian inference interpretation a reasoner is interpreted as modelling the world as if its causal structure is as follows, where each of the `$S$' nodes is associated with the kernel $\phi$.
\begin{equation}
\raisebox{-0.5\height}{\includegraphics[height=5.5em]{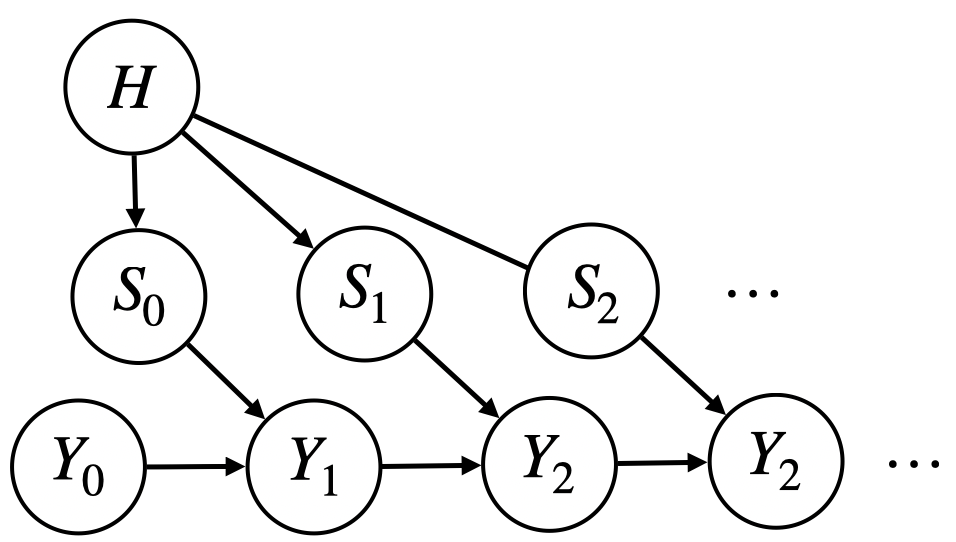}}
\label{bayesian-network-iid}
\end{equation}
However, the true dynamics of the world could still correspond to the Bayesian network in \cref{bayesian-network-context} or some other causal structure.
The reasoner is simply (interpreted as being) unable to perceive the correlations among its inputs.

%

%
%


In \cref{app:examples} we present three examples of machines with consistent Bayesian inference interpretations. 
The first example is a non-deterministic finite machine with three states. 
The consistent Bayesian inference interpretation we provide also involves subjectively impossible observations.

The second example is a countably infinite deterministic machine that counts occurrences of each of two possible observations.
The consistent Bayesian interpretation we provide intuitively considers this machine to be inferring the bias of a coin that is flipped to cause its observations.  
It uses the standard conjugate prior to its model as an interpretation map. 

The third example is also a countably infinite machine with two possible observations.
However, it only stores the difference between the number of the two possible observations. 
Intuitively, instead of considering all possible biases of a coin as the second example the consistent Bayesian interpretation we provide for this machine infers which of two specific biases of a coin is causing its observations.
We also hint at how the second machine can ``inherit'' this interpretation.

\section{Discussion}
We see the main contribution of this work as a conceptual one. 
The consistency equation involved in \cref{bayesian-filtering-interpretation-definition,bayesian-inference-interpretation-definition}
 can be seen either as a constraint on the machines that have a particular interpretation or as a constraint on the interpretations that a particular machine allows. 
Every machine has an interpretation with respect to a trivial model (one that has no parameter), but in order to have an interpretation with respect to a non-trivial model, a machine must at least obey the constraints discussed in \cref{unpacking-filtering}.

Similar consistency equations should exist for approximate Bayesian filtering as well as for related inference problems like Bayesian smoothing or prediction.
Even non-Bayesian normative theories about what a system should be representing and how this should change under external influences will probably have associated consistency equations.

Another direction to extend the concept of consistency equations is to take into account a possible influence of the machine's state on the external world.
On the interpretation side this would mean going beyond perception and representation to also include deliberate actions that combine with beliefs and possibly also with goals.
%
A machine with such an interpretation might deserve the term \emph{agent} instead of reasoner.

Our work is related to other current efforts to capture the notion of agent using category theory.
These include approaches to Bayesian inference \cite{smithe_bayesian_2020} and game theory \cite{bolt_bayesian_2019}.
The idea that agency is related to parametrisation has also arisen in these contexts \cite{capucci_towards_2021,st_clere_smithe_cyber_2021,capucci_translating_2021}.
These works focus on the notion of a \emph{lens} and its generalisations.
It is interesting to note that our notions of interpretation seem to be different, and somewhat simpler, in that they lack the bidirectional nature of lenses.
We conjecture that a more lens-like bidirectional structure would be needed if we were to consider Bayesian smoothing rather than Bayesian filtering.
It will be interesting in future work to better understand the relationship between lens-like categories and the concept of interpretation developed in the present work.

\subsection{Relation to the Free Energy Principle}
\label{sec:FEP}

Let us now consider the relation to the Free Energy Principle (FEP) which is also referred to as active inference.
The relevant part of FEP is the part called ``Bayesian mechanics'' \cite{friston_free_2019,da_costa_bayesian_2021,parr_markov_2020}.
It seems that the ingredients for an interpretation map can be found in this literature:
\cite[eq. 3.3]{da_costa_bayesian_2021} describes a Markov kernel of an appropriate type, as we detail in \cref{app:FEP}.
However, it is not currently clear to us whether the FEP can be formulated in terms of a consistency equation that this kernel obeys.
Presumably, such an equation would be different from our \cref{bayesian-filtering-interpretation-definition,bayesian-inference-interpretation-definition}, because the FEP is concerned with approximate rather than exact inference and deals with continuous time.

One important difference between our approach and current formulations of FEP is that the FEP requires a stationarity assumption on the true dynamics of the agent-environment system.
It seems to us that this is used to derive something that corresponds to a model.
In our approach the reasoner's and the ``ground truth'' dynamics of the environment are different things, and partly for this reason we need no stationarity assumption.
We see this conceptual separation as an advantage of the consistency equation approach, and we believe that by incorporating these ideas it might be possible to formulate the FEP in a way that would make its assumptions clearer and perhaps even avoid the need for the stationarity assumption.
Although we do not currently know the precise relationship between our work and the FEP at a technical level, we explore it in more detail in \cref{app:FEP}.

\bibliographystyle{splncs04}
\bibliography{./bibliography}

\begin{thebibliography}{10}
\providecommand{\url}[1]{\texttt{#1}}
\providecommand{\urlprefix}{URL }
\providecommand{\doi}[1]{https://doi.org/#1}

\bibitem{aguilera_how_2021}
Aguilera, M., Millidge, B., Tschantz, A., Buckley, C.L.: How particular is the
  physics of the {Free} {Energy} {Principle}? arXiv:2105.11203 [q-bio]  (May
  2021), \url{http://arxiv.org/abs/2105.11203}, arXiv: 2105.11203

\bibitem{albantakis_macro_2020}
Albantakis, L., Massari, F., Beheler-Amass, M., Tononi, G.: A macro agent and
  its actions. arXiv:2004.00058 [cs, q-bio]  (Mar 2020),
  \url{http://arxiv.org/abs/2004.00058}, arXiv: 2004.00058

\bibitem{ay_umwelt_2015}
Ay, N., Löhr, W.: The {Umwelt} of an embodied agent--a measure-theoretic
  definition. Theory in Biosciences = Theorie in Den Biowissenschaften
  \textbf{134}(3-4),  105--116 (Dec 2015). \doi{10.1007/s12064-015-0217-3}

\bibitem{baez_physics_2011}
Baez, J., Stay, M.: Physics, {Topology}, {Logic} and {Computation}: {A}
  {Rosetta} {Stone}. In: Coecke, B. (ed.) New {Structures} for {Physics}, pp.
  95--172. Lecture {Notes} in {Physics}, Springer, Berlin, Heidelberg (2011).
  \doi{10.1007/978-3-642-12821-9\_2},
  \url{https://doi.org/10.1007/978-3-642-12821-9\_2}

\bibitem{beer_autopoiesis_2004}
Beer, R.D.: Autopoiesis and {Cognition} in the {Game} of {Life}. Artificial
  Life  \textbf{10}(3),  309--326 (2004). \doi{10.1162/1064546041255539},
  \url{/journal/10.1162/1064546041255539}

\bibitem{beer_cognitive_2014}
Beer, R.D.: The cognitive domain of a glider in the game of life. Artificial
  Life  \textbf{20}(2),  183--206 (2014). \doi{10.1162/ARTL\_a\_00125}

\bibitem{biehl_towards_2016}
Biehl, M., Ikegami, T., Polani, D.: Towards information based spatiotemporal
  patterns as a foundation for agent representation in dynamical systems. In:
  Proceedings of the {Artificial} {Life} {Conference} 2016. pp. 722--729. The
  MIT Press (Jul 2016). \doi{10.7551/978-0-262-33936-0-ch115},
  \url{https://mitpress.mit.edu/sites/default/files/titles/content/conf/alife16/ch115.html}

\bibitem{biehl_dynamics_2020}
Biehl, M., Kanai, R.: Dynamics of a {Bayesian} {Hyperparameter} in a {Markov}
  {Chain}. In: Verbelen, T., Lanillos, P., Buckley, C.L., De~Boom, C. (eds.)
  Active {Inference}. pp. 35--41. Communications in {Computer} and
  {Information} {Science}, Springer International Publishing, Cham (2020).
  \doi{10.1007/978-3-030-64919-7\_5}

\bibitem{biehl_technical_2021}
Biehl, M., Pollock, F.A., Kanai, R.: A {Technical} {Critique} of {Some} {Parts}
  of the {Free} {Energy} {Principle}. Entropy  \textbf{23}(3), ~293 (Mar 2021).
  \doi{10.3390/e23030293}, \url{https://www.mdpi.com/1099-4300/23/3/293},
  number: 3 Publisher: Multidisciplinary Digital Publishing Institute

\bibitem{bolt_bayesian_2019}
Bolt, J., Hedges, J., Zahn, P.: Bayesian open games. arXiv:1910.03656 [cs,
  math]  (Oct 2019), \url{http://arxiv.org/abs/1910.03656}, arXiv: 1910.03656

\bibitem{capucci_towards_2021}
Capucci, M., Gavranović, B., Hedges, J., Rischel, E.F.: Towards foundations of
  categorical cybernetics. arXiv:2105.06332 [math]  (May 2021),
  \url{http://arxiv.org/abs/2105.06332}, arXiv: 2105.06332

\bibitem{capucci_translating_2021}
Capucci, M., Ghani, N., Ledent, J., Forsberg, F.N.: Translating {Extensive}
  {Form} {Games} to {Open} {Games} with {Agency}. arXiv:2105.06763 [cs, math]
  (May 2021), \url{http://arxiv.org/abs/2105.06763}, arXiv: 2105.06763

\bibitem{cho_disintegration_2019}
Cho, K., Jacobs, B.: Disintegration and {Bayesian} {Inversion} via {String}
  {Diagrams}. Mathematical Structures in Computer Science  \textbf{29}(7),
  938--971 (Aug 2019). \doi{10.1017/S0960129518000488},
  \url{http://arxiv.org/abs/1709.00322}, arXiv: 1709.00322

\bibitem{coecke_categories_2011}
Coecke, B., Paquette, {\'E}.: Categories for the {Practising} {Physicist}. In:
  Coecke, B. (ed.) New {Structures} for {Physics}, pp. 173--286. Lecture
  {Notes} in {Physics}, Springer, Berlin, Heidelberg (2011).
  \doi{10.1007/978-3-642-12821-9\_3},
  \url{https://doi.org/10.1007/978-3-642-12821-9\_3}

\bibitem{coecke_picturing_2017}
Coecke, B., Kissinger, A.: Picturing quantum processes: a first course in
  quantum theory and diagrammatic reasoning. Cambridge University Press,
  Cambridge, United Kingdom ; New York, NY, USA (2017)

\bibitem{da_costa_bayesian_2021}
Da~Costa, L., Friston, K., Heins, C., Pavliotis, G.A.: Bayesian {Mechanics} for
  {Stationary} {Processes}. arXiv:2106.13830 [math-ph, physics:nlin, q-bio]
  (Jun 2021), \url{http://arxiv.org/abs/2106.13830}, arXiv: 2106.13830

\bibitem{dennett_true_1981}
Dennett, D.C.: True {Believers} : {The} {Intentional} {Strategy} and {Why} {It}
  {Works}. In: Heath, A.F. (ed.) Scientific {Explanation}: {Papers} {Based} on
  {Herbert} {Spencer} {Lectures} {Given} in the {University} of {Oxford}, pp.
  53--75. Clarendon Press (1981)

\bibitem{fong_invitation_2019}
Fong, B., Spivak, D.I.: An invitation to applied category theory: seven
  sketches in compositionality. Cambridge University Press, Cambridge ; New
  York, NY (2019)

\bibitem{friston_free_2019}
Friston, K.: A free energy principle for a particular physics. arXiv:1906.10184
  [q-bio]  (Jun 2019), \url{http://arxiv.org/abs/1906.10184}, arXiv: 1906.10184

\bibitem{friston_sophisticated_2021}
Friston, K., Da~Costa, L., Hafner, D., Hesp, C., Parr, T.: Sophisticated
  {Inference}. Neural Computation  \textbf{33}(3),  713--763 (Mar 2021).
  \doi{10.1162/neco\_a\_01351}, \url{https://doi.org/10.1162/neco\_a\_01351}

\bibitem{friston_interesting_2020}
Friston, K., Da~Costa, L., Parr, T.: Some interesting observations on the free
  energy principle. arXiv:2002.04501 [q-bio]  (Feb 2020),
  \url{http://arxiv.org/abs/2002.04501}, arXiv: 2002.04501

\bibitem{fritz_synthetic_2020}
Fritz, T.: A synthetic approach to {Markov} kernels, conditional independence
  and theorems on sufficient statistics. Advances in Mathematics  \textbf{370},
   107239 (Aug 2020). \doi{10.1016/j.aim.2020.107239},
  \url{https://www.sciencedirect.com/science/article/pii/S0001870820302656}

\bibitem{jacobs_channel-based_2020}
Jacobs, B.: A channel-based perspective on conjugate priors. Mathematical
  Structures in Computer Science  \textbf{30}(1),  44--61 (Jan 2020).
  \doi{10.1017/S0960129519000082},
  \url{https://www.cambridge.org/core/journals/mathematical-structures-in-computer-science/article/channelbased-perspective-on-conjugate-priors/D7897ABA1AA06E5F586F60CB21BDDB32},
  publisher: Cambridge University Press

\bibitem{jacobs_channel-based_2018}
Jacobs, B.: A {Channel}-{Based} {Perspective} on {Conjugate} {Priors}.
  arXiv:1707.00269 [cs]  (Sep 2018), \url{http://arxiv.org/abs/1707.00269},
  arXiv: 1707.00269

\bibitem{jacobs_finettis_2020}
Jacobs, B., Staton, S.: De {Finetti}’s {Construction} as a {Categorical}
  {Limit}. In: Petrişan, D., Rot, J. (eds.) Coalgebraic {Methods} in
  {Computer} {Science}. pp. 90--111. Lecture {Notes} in {Computer} {Science},
  Springer International Publishing, Cham (2020).
  \doi{10.1007/978-3-030-57201-3\_6}

\bibitem{knill_bayesian_2004}
Knill, D.C., Pouget, A.: The {Bayesian} brain: the role of uncertainty in
  neural coding and computation. Trends in Neurosciences  \textbf{27}(12),
  712--719 (Dec 2004). \doi{10.1016/j.tins.2004.10.007},
  \url{https://www.cell.com/trends/neurosciences/abstract/S0166-2236(04)00335-2},
  publisher: Elsevier

\bibitem{kolchinsky_semantic_2018}
Kolchinsky, A., Wolpert, D.H.: Semantic information, autonomous agency and
  non-equilibrium statistical physics. Interface Focus  \textbf{8}(6),
  20180041 (Dec 2018). \doi{10.1098/rsfs.2018.0041},
  \url{https://royalsocietypublishing.org/doi/full/10.1098/rsfs.2018.0041}

\bibitem{krakauer_information_2020}
Krakauer, D., Bertschinger, N., Olbrich, E., Flack, J.C., Ay, N.: The
  information theory of individuality. Theory in Biosciences  \textbf{139}(2),
  209--223 (Jun 2020). \doi{10.1007/s12064-020-00313-7},
  \url{https://doi.org/10.1007/s12064-020-00313-7}

\bibitem{libby_noisy_2007}
Libby, E., Perkins, T.J., Swain, P.S.: Noisy information processing through
  transcriptional regulation. Proceedings of the National Academy of Sciences
  \textbf{104}(17),  7151--7156 (Apr 2007)

\bibitem{ma_neural_2014}
Ma, W.J., Jazayeri, M.: Neural coding of uncertainty and probability. Annual
  Review of Neuroscience  \textbf{37},  205--220 (2014).
  \doi{10.1146/annurev-neuro-071013-014017}

\bibitem{mcgregor_bayesian_2017}
McGregor, S.: The {Bayesian} stance: {Equations} for ‘as-if’ sensorimotor
  agency. Adaptive Behavior p. 105971231770050 (Mar 2017).
  \doi{10.1177/1059712317700501},
  \url{http://journals.sagepub.com/doi/10.1177/1059712317700501}

\bibitem{nakamura_connection_2021}
Nakamura, K., Kobayashi, T.J.: Connection between the {Bacterial} {Chemotactic}
  {Network} and {Optimal} {Filtering}. Physical Review Letters
  \textbf{126}(12),  128102 (Mar 2021). \doi{10.1103/PhysRevLett.126.128102},
  \url{https://link.aps.org/doi/10.1103/PhysRevLett.126.128102}

\bibitem{orseau_agents_2018}
Orseau, L., McGill, S.M., Legg, S.: Agents and {Devices}: {A} {Relative}
  {Definition} of {Agency}. arXiv:1805.12387 [cs, stat]  (May 2018),
  \url{http://arxiv.org/abs/1805.12387}, arXiv: 1805.12387

\bibitem{parr_markov_2020}
Parr, T., Da~Costa, L., Friston, K.: Markov blankets, information geometry and
  stochastic thermodynamics. Philosophical Transactions of the Royal Society A:
  Mathematical, Physical and Engineering Sciences  \textbf{378}(2164),
  20190159 (Feb 2020). \doi{10.1098/rsta.2019.0159},
  \url{https://royalsocietypublishing.org/doi/full/10.1098/rsta.2019.0159}

\bibitem{risken_fokker-planck_1996}
Risken, H., Frank, T.: The {Fokker}-{Planck} {Equation}: {Methods} of
  {Solution} and {Applications}. Springer {Series} in {Synergetics},
  Springer-Verlag, Berlin Heidelberg, 2 edn. (1996).
  \doi{10.1007/978-3-642-61544-3},
  \url{https://www.springer.com/gp/book/9783540615309}

\bibitem{rosas_causal_2020}
Rosas, F.E., Mediano, P.A.M., Biehl, M., Chandaria, S., Polani, D.: Causal
  {Blankets}: {Theory} and {Algorithmic} {Framework}. In: Verbelen, T.,
  Lanillos, P., Buckley, C.L., De~Boom, C. (eds.) Active {Inference}. pp.
  187--198. Communications in {Computer} and {Information} {Science}, Springer
  International Publishing, Cham (2020). \doi{10.1007/978-3-030-64919-7\_19}

\bibitem{smithe_bayesian_2020}
Smithe, T.S.C.: Bayesian {Updates} {Compose} {Optically}. arXiv:2006.01631
  [math, stat]  (Jul 2020), \url{http://arxiv.org/abs/2006.01631}, arXiv:
  2006.01631

\bibitem{st_clere_smithe_cyber_2021}
St~Clere~Smithe, T.: Cyber {Kittens}, or {Some} {First} {Steps} {Towards}
  {Categorical} {Cybernetics}. Electronic Proceedings in Theoretical Computer
  Science  \textbf{333},  108--124 (Feb 2021). \doi{10.4204/EPTCS.333.8},
  \url{http://arxiv.org/abs/2101.10483v1}

\bibitem{still_thermodynamics_2012}
Still, S., Sivak, D.A., Bell, A.J., Crooks, G.E.: The thermodynamics of
  prediction. {arXiv} e-print 1203.3271 (Mar 2012),
  \url{http://arxiv.org/abs/1203.3271}, phys. Rev. Lett. 109, 120604 (2012)

\bibitem{wikipedia_conjugate_prior}
{Wikipedia contributors}: Conjugate prior --- {Wikipedia}{,} the free
  encyclopedia (2021),
  \url{https://en.wikipedia.org/w/index.php?title=Conjugate_prior&oldid=1030202570},
  [Online; accessed 8-July-2021]

\end{thebibliography}

\appendix

\section{Category-Theoretic Probability and String Diagrams}
\label{string-diagram-intro}

In this paper we use some concepts from category-theoretic probability, and in particular we use a notation known as string diagrams.
A full introduction to these topics would be out of scope of the paper, but we include here an informal introduction to the topic.
We do this because, to our knowledge, no concise introduction currently exists that is focused on (classical) probability and does not assume a background in category theory. 
We assume that the reader knows the definition of a category, but not much more than that.

\Cref{string-diagram-intro-main} introduces the basic concepts, mostly in the context of discrete probability.
In \cref{measure-theory} we briefly comment on how this extends to the general case of measure-theoretic probability with very little extra work.
In \cref{conditionals-and-bayes} we explain how to reason about conditional probabilities and Bayes' theorem within this category-theoretic context.

These sections contain no original material.
Their purpose is to give the reader enough information to be able to read the string diagram equations in the main text and later sections of the appendix without needing to consult a category theory text.
However, they are intended neither as an authoritative technical reference nor as a comprehensive review, and readers should consult the cited references for full details.

\subsection{Introduction to String Diagrams and Category-Theoretic Probability}
\label{string-diagram-intro-main}

A full technical introduction to the use of string diagrams in probability can be found in \cite{fritz_synthetic_2020} or the earlier \cite{cho_disintegration_2019}, but these works require some knowledge of category theory.
The string diagram notation predates its use in probability and has many other applications.
One could consult \cite{baez_physics_2011,coecke_categories_2011,fong_invitation_2019,coecke_picturing_2017} for tutorial introductions to diagrammatic reasoning in other fields, of various different flavours.
Here we present it somewhat informally and only in the context of probability.

It should be kept in mind that, despite our somewhat informal introduction, string diagrams are formal expressions.
The main difference between them and the more familiar kind of mathematical expression formed from strings of symbols is their two-dimensional syntax.
This makes it easier to express certain concepts. (Particularly those relating to joint distributions, in the case of probability.)

We use the so-called \emph{Markov category} approach to probability \cite{fritz_synthetic_2020}.
The main idea here is to express everything in terms of \emph{measurable spaces} and \emph{Markov kernels}, whose definitions we outlined in the main text.%
\footnote{In fact for most of the paper we will work much more abstractly than this. It would be more correct to say ``objects in a Markov category'' wherever we say ``measurable space'' and ``morphisms in a Markov category'' wherever we say ``Markov kernel,'' since for most of the paper we will reason at the category level, and we will not directly invoke the definition of a measurable space. We have chosen to use the more concrete terms because they express a clear intuition for how these objects and morphisms are intended to be interpreted.}
To explain how the framework works, let us consider the special case where the only measurable spaces we are interested in are finite sets (with their power sets as their $\sigma$-algebras).
If $A$ and $B$ are finite sets then a Markov kernel can be thought of as just a function $f\colon A\to P(B)$, where $P(B)$ is the set of all probability distributions over $B$.
(The set $P(B)$ may be thought of as a $(|B|-1)$-dimensional simplex, consisting of all those vectors in $\mathbb{R}^{|B|}$ whose components are all non-negative and sum to 1.)
Such a function amounts to a $|B|$-by-$|A|$ stochastic matrix, although some care needs to be taken over which rows correspond to which elements of $B$ and which columns to which elements of $A$.

In this finite case, we write $f(b\mg a)$ to denote the probability that the kernel $f$ assigns to the outcome $b\in B$ when given the input $a\in A$.
We use a thick vertical line to indicate a close relationship to conditional probability while also emphasising that the concept is different:
given a kernel $f\colon A\to P(B)$ the quantities $f(b\mg a)$ are always defined, regardless of whether any probability distribution has been defined over $A$, and regardless of whether $a$ has a nonzero probability according to such a distribution.
More common notations include $|$~or~$;$ in place of $\mg$.

We also write $f(a)$ for the probability distribution over $B$ that the function $f$ returns when given the input $a$.
We could say that $f(b\mg a)$ is defined as $f(a)(b)$.

Given Markov kernels $f\colon A\to P(B)$ and $g\colon B\to P(C)$, we can compose them to form a new kernel of type $A\to P(C)$.
We write this $f\comp g$.
It is given by
\begin{equation}
\label{finite_kernel_composition}
 (f\comp g)(c\mg a) = \sum_{b\in BY} f(b\mg a)\,g(c\mg b).
\end{equation}
In this finite case this is simply matrix multiplication, and we could have denoted it $gf$ instead of $f\comp g$ accordingly.
(Another common notation is $g\circ f$.)
We prefer $f\comp g$ because it puts $f$ and $g$ in the same order that they will appear in string diagrams.

It is straightforward to show that composition is associative, that is
\begin{equation}
\label{associativity-law}
(f\comp g)\comp h = f\comp (g\comp h).
\end{equation}
In addition, for every finite set $A$ there is an identity kernel, which amounts to just the $|A|$-by-$|A|$ identity matrix.
We write this as $\id_A$ and define it by $\id_A(a'\mg a) = \delta_{a,a'}$.
For every Markov kernel $f\colon A\to P(B)$ we have
\begin{equation}
\label{identity-law}
 \id_A\comp f = f = f\comp \id_B.
\end{equation}

These two facts mean that there is a {category} whose objects are finite sets and whose morphisms are Markov kernels between finite sets.
This category is called $\mathsf{FinStoch}$.

Since Markov kernels are morphisms in a category, we will often write $f\colon A \todot B$ instead of $f\colon A\to P(B)$, using the dotted arrow $\todot$ to distinguish morphisms in $\mathsf{FinStoch}$ and related categories from ordinary functions.
(In the main text we continue writing them as functions in order to avoid introducing new notation.)

The composition of Markov kernels can be generalised to the case of measure-theoretic probability, which allows us to reason about continuous probability and more general probability measures using the same kinds of diagram and much of the same reasoning.
We briefly discuss this in more detail in \cref{measure-theory}.
The main difference is that composition becomes integration over measures rather than summation.
%
%

Probability measures themselves may be seen as a special case of Markov kernels.
Consider a set with a single element, denoted $\one = \{\star\}$.
(The identity of the element does not matter because all one-element sets are isomorphic to each other.
Category theorists often speak of ``the one-element set'' for this reason. We use a star to denote the element.)
Then a Markov kernel $p\colon \one\todot A$ is a function $p\colon \one\to P(A)$, which takes an element of $\one$ and returns a probability measure over $A$.
Since there is only one element of $\one$ this means that the kernel $p$ only defines a single probability measure over~$A$.
We therefore think of Markov kernels $\one\todot A$ and probability measures over $A$ as essentially the same thing.
%

%

We now begin to introduce the string diagram notation.
A Markov kernel $f\colon A\todot B$ will be denoted
\begin{equation}
\pbox{
\stringdiagram {
 \wire (0,0) -- (4,0);
 \stoch{2,0}{$f$}{3};
 \ulabel{0,0}{A};
 \ulabel{4,0}{B};
 }
 }{.}
\end{equation}
This expression means much the same thing as the notation $f\colon A \todot B$.
It is just a formal symbol denoting the kernel $f$, annotated with type information.
%

The composition of kernels $f\colon A\todot B$ and $g\colon B\todot C$ is written
\begin{equation}
\label{composition-string-diagram}
\pbox{
\stringdiagram {
 \wire (0,0) -- (4,0);
 \stoch{2,0}{$f\comp g$}{3};
 \ulabel{0,0}{A};
 \ulabel{4,0}{C};
 }
 \quad=\quad
\stringdiagram {
 \wire (0,0) -- (7,0);
 \stoch{2,0}{$f$}{3};
 \ulabel{0,0}{A};
 \ulabel{3.5,0}{B};
 \stoch{5,0}{$g$}{3};
 \ulabel{7,0}{C};
 }
 }{.}
\end{equation}
The left and right hand side of this equation are just two different ways to write the composite kernel $f\comp g$, as defined by \cref{finite_kernel_composition} or its measure-theoretic generalisation.

In reading a diagram like the right-hand side of \cref{composition-string-diagram} we find it helpful to imagine an element of $A$ travelling along the wire from the left.
As it passes through the kernel $f$ it is stochastically transformed into an element of $B$, in a way that might depend on its original value.
It then travels further to the right and is stochastically transformed by $g$ into an element of $C$.
\Cref{finite_kernel_composition} can be seen as describing this process.

In string diagrams a special notation is used for identity kernels (or identity morphisms more generally): an identity kernel $\id_A$ is drawn simply as a wire with no box on it,
\begin{equation}
\pbox{
 \stringdiagram {
 \wire (0,0) -- (4,0);
 \ulabel{0,0}{A};
 }
 }{.}
\end{equation}
For any Markov kernel $f\colon A\todot B$ the identity law \cref{identity-law} can then be written
\begin{equation}
\begin{aligned}
\stringdiagram {
 \wire (0,0) -- (4,0);
 \stoch{2,0}{$f$}{3};
 \ulabel{0,0}{A};
 \ulabel{4,0}{B};
 }
 & \quad=\quad \stringdiagram {
 \wire (-2,0) -- (4,0);
 \stoch{2,0}{$f$}{3};
 \ulabel{-2,0}{A};
 \ulabel{4,0}{B};
 } \\[1em]
&\quad=\quad
\pbox{\stringdiagram {
 \wire (0,0) -- (6,0);
 \stoch{2,0}{$f$}{3};
 \ulabel{0,0}{A};
 \ulabel{6,0}{B};
 }
 }{.}
\end{aligned}
\end{equation}
This allows us to think of the wires as stretchy: we can extend and contract them at will.
We will think of the wires as continuously deformable, rather than extending and contracting in discrete units.
This is justified by the formal theory of string diagrams.
(One may informally think of the wire itself as an infinite chain of identity kernels, all composed together.)
This ability to continuously deform diagrams turns out to be an extremely powerful and useful idea.

Another special notation is used for one-element sets\footnote{or in a more general context, the unit object of a monoidal category}:
they are drawn as no wire at all.
%
%
For this reason a probability measure over $A$, that is, a kernel $p\colon \one\todot A$, is drawn as
\begin{equation}
\pbox{
\stringdiagram {
 \wire (2,0) -- (4,0);
 \stoch{2,0}{$p$}{3};
 \ulabel{4,0}{A};
 }
 }{.}
\end{equation}
(Morphisms of this kind are sometimes known as ``states,'' and they are often drawn as a triangle rather than a box, though here we draw them in the same style as other morphisms.)

It is worth noting that the kernels $p$ and $f$ above can be composed, yielding
\begin{equation}
\pbox{
\stringdiagram {
 \wire (2,0) -- (4,0);
 \stoch{2,0}{$p\comp f$}{3};
 \ulabel{4,0}{B};
 }
 \quad=\quad
\stringdiagram {
 \wire (2,0) -- (7,0);
 \stoch{2,0}{$p$}{3};
 \ulabel{3.5,0}{A};
 \stoch{5,0}{$f$}{3};
 \ulabel{7,0}{B};
 }
 }{.}
\end{equation}
%
Because of this, although the kernel $f\colon A\todot B$ is defined as a function $f\colon A\to P(B)$ mapping \emph{elements} of $A$ to probability distributions over $B$, we can instead choose to see it as mapping \emph{probability measures} over $A$ to probability measures over $B$.
In the finite case, if we think of finite probability distributions as normalised and nonnegative vectors in $\mathbb{R}^{n}$, then $f$ can be seen as a linear map with the property that it maps points in one simplex to points in another.
(This justifies thinking of it as a stochastic matrix.)

The string diagram notation becomes useful when we start thinking about joint distributions.
We do this by drawing wires in parallel.
As an example, we can consider a Markov kernel defined by a function $h\colon A\times B \to P(C\times D)$.
This function takes two arguments, an element of $A$ and an element of $B$, and it returns a joint probability distribution over $C$ and $D$.
In string diagrams we write this as
\begin{equation}
\pbox{
 \stringdiagram {
 \wire (0,0) -- (4,0);
 \wire (0,2) -- (4,2);
 \stoch{2,1}{$h$}{4};
 \ulabel{0,2}{B};
 \ulabel{0,0}{A};
 \ulabel{4,2}{D};
 \ulabel{4,0}{C};
 }
}{.} 
\end{equation}
In symbols, we write $h\colon A\otimes B \todot C\otimes D$. 
An object like $A\otimes B$, drawn as two parallel wires, can either be thought of as the measurable space $A\times B$ (which is the Cartesian product of sets in the finite case), or as the space of probability measures over $A\times B$.
The symbol $\otimes$ is referred to as a monoidal product.

There is some inherent ambiguity in this notation.
If we draw three parallel wires,
$
\stringdiagram{
\wire (0,0) -- (3,0);
\wire (0,1.6) -- (3,1.6);
\wire (0,3.2) -- (3,3.2);
\llabel{0,0}{A};
\llabel{0,1.6}{B};
\llabel{0,3.2}{C};
}\,,
$
it could either mean $(A\otimes B)\otimes C$ or $A\otimes (B\otimes C)$.
In the finite case, these correspond to the sets $(A\times B)\times C$ and $A\times (B\times C)$. 
These are different sets, since one is composed of pairs $((a,b),c)$ and the other of pairs $(a,(b,c))$.
This ambiguity is not important in practice, however, and the formal machinery of \emph{monoidal categories} allows us to use string diagrams without worrying about it.
We do not give a formal treatment of this here.
(A concise summary can be found in \cite{baez_physics_2011}.)
Instead we simply remark that when we draw three parallel wires we think of joint distributions over $A$, $B$ and $C$, and the precise distinction between $P(A\times (B\times C))$ and $P((A\times B)\times C)$ will not be important to us.

In a similar vein, the spaces $A$ and $A\otimes \one$ are different, but the difference is not important to us, and in fact they are written the same way in string diagrams.
This is because we draw $\one$ as an invisible wire.
This also allows us to write
\begin{equation}
 \pbox{
 \stringdiagram{
 	\wire (0,0) -- (4,0);
	\wire (0,1) -- (4,1);
	\llabel{0,-0.1}{A};
	\llabel{0,1.1}{B};
 }
 \quad=\quad
 \stringdiagram{
 	\wire (0,0) -- (4,0);
	\wire (0,2.5) -- (4,2.5);
	\llabel{0,0}{A};
	\llabel{0,2.5}{B};
 }
 \quad=\quad
 \stringdiagram{
 	\wire (0,0) -- (2,-1) -- (4,0);
	\wire (0,1) -- (1,1) -- (3,2) -- (4,2);
	\llabel{0,-0.1}{A};
	\llabel{0,1.1}{B};
 }
 }{.}
\end{equation}
That is, string diagrams are stretchy in the vertical direction as well as the horizontal one.
We can bend the wires, as long as we don't deform them so much that they point backwards, from right to left.

This also allows to write things like
\begin{equation}
\stringdiagram{
	\wire (0,1) -- (2,1);
	\ulabel{0,1}{$A$};
	
	\wire (2,2) -- (4,2);
	\ulabel{4,2}{$C$};
	
	\wire (2,0) -- (4,0);
	\ulabel{4,0}{$B$};
	
	\stoch{2,1}{$f$}{4};
}
\end{equation}
for a kernel $f\colon A\todot B\otimes C$.

We can also draw morphisms (i.e. Markov kernels) in parallel with each other, for example,
\begin{equation}
\pbox{
 \stringdiagram {
 \wire (0,0) -- (4,0);
 \wire (0,3) -- (4,3);
 \stoch{2,0}{$f$}{3};
 \stoch{2,3}{$g$}{3};
 \ulabel{0,3}{C};
 \ulabel{0,0}{A};
 \ulabel{4,3}{D};
 \ulabel{4,0}{B};
 }
}{.} 
\end{equation}
We write this in symbols as $f\otimes g$, which is a morphism of type $A\otimes C \todot B\otimes D$.
In the finite case, it is given by
\begin{equation}
 (f\otimes g)(b,d\mg a,c) = f(b\mg a) \, g(d\mg c).
\end{equation}
The probabilities $f(b\mg c)$ and $g(d\mg c)$ are multiplied together because the two Markov kernels are operating in parallel.
One can imagine an element of $A$ entering from the bottom left and being stochastically transformed by $f$ into an element of $B$, while in parallel, and independently, an element of $C$ enters from the top left and is stochastically transformed by $g$ into an element of $D$.
In general, in the finite case, $f\otimes g$ is given by the tensor product of the stochastic matrices that represent $f$ and $g$.
(This might give some intuition for the symbol~$\otimes$.)

We can cross wires over each other.
(In category theory terms, the categories we are concerned with are symmetric monoidal categories.)
The diagram
\begin{equation}
\stringdiagram{
 \wire (0,0) -- (1,0) -- (3,2) -- (4,2);
 \wire (0,2) -- (1,2) -- (3,0) -- (4,0);
 \llabel{0,0}{A};
 \llabel{0,2}{B};
 \rlabel{4,0}{B};
 \rlabel{4,2}{A};
}
\end{equation}
can be seen as a Markov kernel $A\otimes B\todot B\otimes A$.
In the finite case it is defined by
\begin{equation}
 \operatorname{swap}_{A,B}(b',a'\mg a,b) = \delta_{a,a'}\delta_{b,b'}.
\end{equation}

We have a number of equations that are standard in monoidal category theory, and allow us to freely slide boxes along wires and bend wires to cross over each other.
These can either be shown directly from the definitions above or (perhaps more usefully) deduced from the definition of a symmetric monoidal category.
Three such equations are as follows.
More details can be found in the references cited above.
\begin{align}
 \stringdiagram{
 \wire (0,0) -- (1,0) -- (3,2) -- (5,2) -- (7,0) -- (8,0);
 \wire (0,2) -- (1,2) -- (3,0) -- (5,0) -- (7,2) -- (8,2);
 \llabel{0,0}{A};
 \llabel{0,2}{B};
 \rlabel{8,0}{A};
 \rlabel{8,2}{B};
}
\quad&=\quad
 \stringdiagram{
 \wire (0,0) -- (8,0);
 \wire (0,2) -- (8,2);
 \llabel{0,0}{A};
 \llabel{0,2}{B};
 \rlabel{8,0}{A};
 \rlabel{8,2}{B};
}
\\[1em]
 \stringdiagram {
 \wire (0,0) -- (8,0);
 \wire (0,2) -- (8,2);
 \stoch{2.25,0}{$f$}{3};
 \stoch{5.75,2}{$g$}{3};
 \llabel{0,2}{C};
 \llabel{0,0}{A};
 \rlabel{8,2}{D};
 \rlabel{8,0}{B};
 }
 \quad&=\quad
 \stringdiagram {
 \wire (0,0) -- (8,0);
 \wire (0,2) -- (8,2);
 \stoch{5.75,0}{$f$}{3};
 \stoch{2.25,2}{$g$}{3};
 \llabel{0,2}{C};
 \llabel{0,0}{A};
 \rlabel{8,2}{D};
 \rlabel{8,0}{B};
 }
 \\[1em]
 \stringdiagram {
 \wire (0,2) -- (5,2) -- (7,0) -- (8,0);
 \wire (0,0) -- (5,0) -- (7,2) -- (8,2); 
 \stoch{2.25,0}{$f$}{3};
 \llabel{0,2}{C};
 \llabel{0,0}{A};
 \rlabel{8,2}{B};
 \rlabel{8,0}{C};
 }
 \quad&=\quad
 \stringdiagram {
 \wire (0,0) -- (1,0) -- (3,2) -- (8,2);
 \wire (0,2) -- (1,2) -- (3,0) -- (8,0);
 \stoch{5.75,2}{$f$}{3};
 \llabel{0,2}{C};
 \llabel{0,0}{A};
 \rlabel{8,2}{B};
 \rlabel{8,0}{C};
 }
\end{align}

So far, everything we have said about string diagrams applies to any symmetric monoidal category.
However, there are two additional things we can add that take us much closer to probability theory.
These are the ability to \emph{copy} and to \emph{delete}.
These operations, and their special properties, do not necessarily exist in other contexts, such as quantum mechanics.
This is a central point of \cite{baez_physics_2011,coecke_categories_2011}.
We will stick to the context of classical probability, however, so copying and deletion will always be possible in this paper.

We cover deletion first.
For every measurable space $A$ there is a unique kernel of type $A\to \one$, which we call $\del_A$.
In the finite case it is given by $\del_A(\star\mg a) = 1$ for all $a\in A$.
We can think of this as a $1\times |A|$ matrix (i.e.\ a row vector) whose entries are all 1.
This is the only possible $1\times |A|$ stochastic matrix.

In string diagrams we write such a deletion kernel as a black dot:
\begin{equation}
\pbox{
\stringdiagram{
 \wire (0,0) -- (2,0);
 \blackdot{2,0};
 \llabel{0,0}{A};
 }
 }{\,.}
\end{equation}
There is one such morphism for every measurable space, but we denote them all with the same kind of black dot.
These black dots have the property that
\begin{equation}
\label{naturality-of-delete}
\pbox{
\stringdiagram{
 \wire (0,0) -- (4.5,0);
 \stoch{2,0}{$f$}{3};
 \ulabel{3.5,0}{B};
 \blackdot{4.5,0};
 \ulabel{0,0}{A};
 }
\,\,=
\stringdiagram{
 \wire (0,0) -- (2,0);
 \blackdot{2,0};
 \ulabel{0,0}{A};
 }
 }{}
\end{equation}
for every Markov kernel $f$. This says that if we take some input $A$, perform some stochastic operation $f$ on it and then delete the result, this is the same as simply deleting the input.\footnote{In category theory terms, this means that the set of all delete kernels collectively forms a natural transformation. (Specifically, it is a natural transformation from the identity functor to the functor that sends all objects to $\one$ and all morphisms to $\id_\one$.) For this reason this property of delete kernels is called ``naturality.''}

The second special operation is copying.
For every measurable space $A$ there is a kernel $\cpy_A\colon A\to A\otimes A$, which we will describe shortly.
We write this also as a black dot, but this time with two output wires rather than one.
\begin{equation}
\pbox{
\stringdiagram{
 \wire (0,0) -- (2,0);
 \wire[1] (2,0) -- (2,1) -- (4,1);
 \wire[1] (2,0) -- (2,-1) -- (4,-1);
 \blackdot{2,0};
 \llabel{0,0}{A};
 \rlabel{4,1}{A};
 \rlabel{4,-1}{A};
 }
 }{\,.}
\end{equation}
Informally, this kernel takes an outcome $a\in A$ and copies it, producing a pair $(a,a)$ of identical values.
It's important to note that it copies \emph{values} rather than \emph{distributions}.
Its output does not consist of two independent and identically distributed elements of $A$ but rather two perfectly correlated elements of $A$ that always have the same value.
In the discrete case the copy map is defined as 
\begin{equation}
 \cpy_A(a'',a'\mg a) = 
\begin{cases}
 1 & \text{if $a''=a'=a$} \\
 0 & \text{otherwise.}
\end{cases}
\end{equation}

In addition to \cref{naturality-of-delete}, the copy and delete maps obey the following properties \cite[definition 2.1]{fritz_synthetic_2020}:
\begin{equation}
\label{coassociativity}
 \stringdiagram{
 \wire (0,0) -- (2,0);
 \wire[1] (2,0) -- (2,1) -- (4,1);
 \wire[1] (2,0) -- (2,-1) -- (6,-1);
 \blackdot{2,0};
 \llabel{0,0}{A};
 \blackdot{4,1};
 \wire[1] (4,1) -- (4,0) -- (6,0);
 \wire[1] (4,1) -- (4,2) -- (6,2);
 \rlabel{6,-1}{A};
 \rlabel{6,0}{A};
 \rlabel{6,2}{A};
 }
 \quad=\,\,\,
\stringdiagram{
 \wire (0,0) -- (2,0);
 \wire[1] (2,0) -- (2,1) -- (6,1);
 \wire[1] (2,0) -- (2,-1) -- (4,-1);
 \blackdot{2,0};
 \llabel{0,0}{A};
 \blackdot{4,-1};
 \wire[1] (4,-1) -- (4,0) -- (6,0);
 \wire[1] (4,-1) -- (4,-2) -- (6,-2);
 \rlabel{6,1}{A};
 \rlabel{6,0}{A};
 \rlabel{6,-2}{A};
 }
\end{equation}

\begin{equation}
\label{cancellation}
\stringdiagram{
 \wire (0,0) -- (2,0);
 \wire[1] (2,0) -- (2,1) -- (3.5,1);
 \wire[1] (2,0) -- (2,-1) -- (4,-1);
 \blackdot{2,0};
 \blackdot{3.5,1};
 \llabel{0,0}{A};
 \rlabel{4,-1}{A};
 } 
 \quad=\,\,\,
\stringdiagram{
 \wire (0,0) -- (4,0);
 \llabel{0,0}{A};
 } 
 \quad=\,\,\,
 \stringdiagram{
 \wire (0,0) -- (2,0);
 \wire[1] (2,0) -- (2,-1) -- (3.5,-1);
 \wire[1] (2,0) -- (2,1) -- (4,1);
 \blackdot{2,0};
 \blackdot{3.5,-1};
 \llabel{0,0}{A};
 \rlabel{4,1}{A};
 } 
\end{equation}

\begin{equation}
 \label{cocommutativity}
 \stringdiagram{
 \wire (0,0) -- (2,0);
 \wire[1] (2,0) -- (2,1) -- (3,1);
 \wire[0.75] (3,1) -- (3.5,1) -- (4.5,-1)--(5,-1);
 \wire[1] (2,0) -- (2,-1) -- (3,-1);
 \wire[0.75] (3,-1) -- (3.5,-1) -- (4.5,1)--(5,1);
 \blackdot{2,0};
 \llabel{0,0}{A};
 \rlabel{5.3,1}{A};
 \rlabel{5.3,-1}{A};
 }
 \quad=\,\,\,
 \stringdiagram{
 \wire (0,0) -- (2,0);
 \wire[1] (2,0) -- (2,1) -- (4,1);
 \wire[1] (2,0) -- (2,-1) -- (4,-1);
 \blackdot{2,0};
 \llabel{0,0}{A};
 \rlabel{4,1}{A};
 \rlabel{4,-1}{A};
 }
\end{equation}

\begin{equation}
\label{delete-compatibility}
 \stringdiagram{
 \wire (0,0) -- (2,0);
 \blackdot{2,0};
 \llabel{0,0}{A\otimes B\quad\quad\,};
 }
 \quad=\,\,\,
 \stringdiagram{
 \wire (0,0) -- (2,0);
 \blackdot{2,0};
 \llabel{0,0}{A};
 \wire (0,2) -- (2,2);
 \blackdot{2,2};
 \llabel{0,2}{B};
 }
\end{equation}
\begin{equation}
\label{copy-compatibility}
  \stringdiagram{
 \wire (0,0) -- (2,0);
 \wire[1.5] (2,0) -- (2,1.5) -- (4,1.5);
 \wire[1.5] (2,0) -- (2,-1.5) -- (4,-1.5);
 \blackdot{2,0};
 \llabel{0,0}{A\otimes B\quad\quad\,};
 \rlabel{4,1.5}{\quad\quad\,A\otimes B};
 \rlabel{4,-1.5}{\quad\quad\,A\otimes B};
 }
 \quad=\,\,\,
\stringdiagram{
 \wire (0,0) -- (2,0);
 \wire[1.5] (2,0) -- (2.5,2) -- (4,2);
 \wire[1.5] (2,0) -- (2.5,-2) -- (4,-2);
 \blackdot{2,0};
 \llabel{0,0}{A};
 \rlabel{4,2}{A};
 \rlabel{4,-2}{A};
 \wire (0,1.5) -- (2,1.5);
 \wire[1.5] (2,1.5) -- (2.5,3.5) -- (4,3.5);
 \wire[1.5] (2,1.5) -- (2.5,-0.5) -- (4,-0.5);
 \blackdot{2,1.5};
 \llabel{0,1.5}{B};
 \rlabel{4,3.5}{B};
 \rlabel{4,-0.5}{B};
 }
\end{equation}

\Cref{coassociativity} says that if we make multiple copies of something it doesn't matter which order we make them in.
\Cref{cancellation} says that if we copy something and then delete one of the copies, that is the same as doing nothing to it.
\Cref{cocommutativity} says that if we copy something and then swap the copies it makes no difference.
(Because the two copies are the same as each other.)

\Cref{delete-compatibility,copy-compatibility} are more technical.
They say that if we have elements of $A$ and $B$ we can delete or copy them as a single element of $A\otimes B$ or separately, as elements of $A$ and $B$, and these should give the same result.

These equations can be derived from the definitions we have given for the finite case.
They may also be derived in various more general measure-theoretic contexts \cite{cho_disintegration_2019,fritz_synthetic_2020}.

However, the approach of \cite{fritz_synthetic_2020} is instead to treat them as \emph{axioms}: any symmetric monoidal category with copy and delete maps that obey \cref{naturality-of-delete,coassociativity,cancellation,cocommutativity,delete-compatibility,copy-compatibility} is called a \emph{Markov category}.
One can do a surprising amount of reasoning about probability theory using these axioms alone, although there are also Markov categories that do not directly resemble the category of measurable spaces and Markov kernels that we have described.
There are various additional axioms that can to be added as well, which then allow more specific results to be proven.
(See \cite{fritz_synthetic_2020} for the details.) 

An important thing to note about the copy operator is that, in general,
\begin{equation}
\pbox{
 \stringdiagram{
 \wire (-2,0) -- (2,0);
 \wire[1.5] (2,0) -- (2,1.5) -- (4,1.5);
 \wire[1.5] (2,0) -- (2,-1.5) -- (4,-1.5);
 \blackdot{2,0};
 \ulabel{1.6,0}{B}
 \llabel{-2,0}{A};
 \stoch{0,0}{$f$}{3};
 \rlabel{4,1.5}{B};
 \rlabel{4,-1.5}{B};
 }
 \quad\ne\,\,\,
 \stringdiagram{
 \wire (0,0) -- (2,0);
 \wire[1.5] (2,0) -- (2,1.5) -- (6,1.5);
 \wire[1.5] (2,0) -- (2,-1.5) -- (6,-1.5);
 \blackdot{2,0};
 \ulabel{1.6,0}{A}
 \llabel{0,0}{A};
 \stoch{4.2,1.5}{$f$}{3};
 \stoch{4.2,-1.5}{$f$}{3};
 \rlabel{6,1.5}{B};
 \rlabel{6,-1.5}{B};
 } 
 }{\,\,.}
\end{equation}
That is, copying the output of a kernel $f$ is not the same as copying its input and then applying two copies of the kernel to it.
Intuitively, this is because $f$ might be stochastic.
If we copy the output we end up with two perfectly correlated copies, whereas if we copy the input then the stochastic variations will be independent.

However, if the kernel is deterministic then copying its input is indeed the same as copying its output.
In fact, in the Markov category framework this is the \emph{definition} of a deterministic Markov kernel: we say a kernel $h\colon A \to B$ is deterministic if
\begin{equation}
\label{determinism}
\pbox{
 \stringdiagram{
 \wire (-2,0) -- (2,0);
 \wire[1.5] (2,0) -- (2,1.5) -- (4,1.5);
 \wire[1.5] (2,0) -- (2,-1.5) -- (4,-1.5);
 \blackdot{2,0};
 \ulabel{1.6,0}{B}
 \llabel{-2,0}{A};
 \sqbox{0,0}{$h$}{3};
 \rlabel{4,1.5}{B};
 \rlabel{4,-1.5}{B};
 }
 \quad=\,\,\,
 \stringdiagram{
 \wire (0,0) -- (2,0);
 \wire[1.5] (2,0) -- (2,1.5) -- (6,1.5);
 \wire[1.5] (2,0) -- (2,-1.5) -- (6,-1.5);
 \blackdot{2,0};
 \ulabel{1.6,0}{A}
 \llabel{0,0}{A};
 \sqbox{4.2,1.5}{$h$}{3};
 \sqbox{4.2,-1.5}{$h$}{3};
 \rlabel{6,1.5}{B};
 \rlabel{6,-1.5}{B};
 } 
 }{\,\,.}
\end{equation}
In this paper we use square boxes for kernels that are known to be deterministic, and boxes with rounded edges for general, possibly-stochastic kernels.

In the main text, we write Markov kernels as functions $f\colon A\to P(B)$, and we write deterministic kernels as functions $f\colon A\to B$.
To be more precise, a deterministic kernel should really also be considered as a function $f\colon A\to P(B)$, such that \cref{determinism} is obeyed.
However, if we assume we are working in a category called $\mathsf{BorelStoch}$ (which is a common assumption in category-theoretic probability) then \cref{determinism} implies that $f$ always returns a delta measure \cite[example 10.5]{fritz_synthetic_2020}, and in this case there is not much harm in treating a deterministic kernel $f$ as a function $f\colon A\to B$.

\subsection{The extension to measure theory}
\label{measure-theory}

Above we described the category $\mathsf{FinStoch}$ and introduced string diagrams mostly in that context.
Here we briefly describe how this generalises to the measure-theoretic case, which is needed in order to think about continuous probability.

In the measure-theoretic case the objects ($X$, $Y$, etc.) are any measurable spaces rather than only finite sets.
Markov kernels are still functions $f\colon X\to P(Y)$, but now $P(Y)$ is the set of all probability measures on the measurable space $Y$.
(That is, $P(Y)$ is the set of all functions from the $\sigma$-algebra associated with $Y$ to $[0,1]$, such that Kolmogorov's axioms are obeyed.)
$P(Y)$ can itself be made into a measurable space in a standard way, and the function $f$ must obey an additional restriction that it be a measurable function.
(This means that the preimage of every element of $P(Y)$ must be a member of the $\sigma$-algebra associated with $X$.)

In this case $f(x)$ is a probability measure rather than a probability distribution, and composition is given by integration rather than summation.
(See \cite[example 4]{fritz_synthetic_2020} for the details.)
This gives rise to a category called $\mathsf{Stoch}$, whose objects are all measurable spaces and whose morphisms are all Markov kernels.
(This category is also known as the Kleisli category of the Giry monad, for reasons we do not discuss here.)

Unfortunately the category $\mathsf{Stoch}$ does not have all of the properties that one might want it to have. (See \cref{conditionals-and-bayes} below.)
Because of this a common approach is to work in a category called $\mathsf{BorelStoch}$ (also discussed in \cite[example 4]{fritz_synthetic_2020}), in which the objects are a subset of measurable spaces called standard Borel spaces, and the morphisms are all Markov kernels between standard Borel spaces.
Standard Borel spaces include many kinds of measurable space that one would be likely to use in practice, and in particular they include both finite sets and $\mathbb{R}^{n}$ with its usual $\sigma$-algebra.

\nvt{In the present paper, the properties of $\mathsf{BorelStoch}$ are used in two ways.
Firstly, in $\mathsf{BorelStoch}$ we can always use conditionals, as explained in the next section.
Secondly,}
%
as a notational convenience we treat deterministic kernels and measurable functions as interchangeable, which makes sense in $\mathsf{BorelStoch}$ but doesn't hold in the more general case of $\mathsf{Stoch}$.

\subsection{Conditionals and Bayes' theorem}
\label{conditionals-and-bayes}

Conditional probabilities and Bayes' theorem play central roles in the theory of inference. Here we briefly discuss how they look in string diagrams.
Given a joint distribution $\!\!\stringdiagram{
\wire (0,0) -- (2,0);
\wire (0,2) -- (2,2);
\stoch{0,1}{$q$}{2.5};
\rlabel{2,0}{A};
\rlabel{2,2}{B};
}
$
we may want to split it up into a product of a marginal and a conditional, which in traditional notation, in the discrete case, would be written $p(a,b) = p(a)\, p(b\,|\, a)$.

The category-theoretic approach, as set out in \cite{cho_disintegration_2019,fritz_synthetic_2020}, is slightly different. We write the following, which is called a \emph{disintegration} of $q$.
\nvt{(The term ``disintegration'' is used because it is the opposite of integration.)}
\begin{equation}
\label{conditional-definition-unparametrised}
  \stringdiagram{
 \stoch{2,0}{$q$}{4};
 \wire(2,1) -- (4.5,1);
 \wire(2,-1) -- (4.5,-1);
 \rlabel{4.5,1}{B};
 \rlabel{4.5,-1}{A};
 }
\,\, =
 \pbox{
  \stringdiagram{
 \stoch{2,0}{$q$}{4};
 \wire(2,1) -- (4,1);
 \wire(2,-1) -- (10,-1);
 \uulabel{4,1}{B};
 \blackdot{4,1};
 \wire[2] (5.5,-1)--(5.5,1)--(10,1);
 \blackdot{5.5,-1};
 \ulabel{5,-1}{A};
 \stoch{8,1}{$c$}{3};
 \rlabel{10,1}{B};
 \rlabel{10,-1}{A};
 }}{.}
\end{equation}
Here, $\!\!\stringdiagram{
\wire (0,0) -- (2,0);
\wire (0,2) -- (1.7,2);
\stoch{0,1}{$q$}{2.5};
\rlabel{2,0}{A};
\blackdot{1.7,2}
}
$
is the marginal of $A$ according to the joint distribution $q$.
In the finite case it can be written $\sum_{b\in B} q(a,b)$.
The kernel  $\!\!\!\stringdiagram{
\wire (0,1)--(4,1);
\stoch{2,1}{$c$}{2}
\llabel{0,1}{A};
\rlabel{4,1}{B};
}$
is called a \emph{conditional} of $p$.
It is defined by \cref{conditional-definition-unparametrised}, which in the finite case can be written 
\begin{equation}
\label{conditional-discrete}
 q(a,b) = \left(\sum_{b'\in B} q(a,b')\right) c(b\mg a).
\end{equation}

This is closely analogous to the identity $p(a,b) = p(a)\, p(b\,|\, a)$.
The difference is that $p(b\,|\,a)$ is defined as $p(a,b)/p(a)$, and is only defined when $p(a)>0$.
On the other hand, in \cref{conditional-discrete}, if $\left(\sum_{b'\in B} q(a,b')\right) = 0$ for some $a\in A$ then $q(a,b)$ must be $0$ for all $b\in B$, and consequently the equation puts no constraint on $c(b\mg a)$ in this case.

This means that instead of being undefined in this case, the conditional $c$ is not \emph{uniquely} defined: there may be many different kernels $c$ that satisfy the equation.

This carries over to the general measure-theoretic case as well.
If we are in the category $\mathsf{BorelStoch}$ then for any joint distribution $\!\!\stringdiagram{
\wire (0,0) -- (2,0);
\wire (0,2) -- (2,2);
\stoch{0,1}{$q$}{2.5};
\rlabel{2,0}{A};
\rlabel{2,2}{B};
}
$ there exists at least one conditional $\!\!\!\stringdiagram{
\wire (0,1)--(4,1);
\stoch{2,1}{$c$}{2}
\llabel{0,1}{A};
\rlabel{4,1}{B};
}$
that satisfies \cref{conditional-definition-unparametrised}, but there might be many.
(In the case of $\mathsf{Stoch}$ conditionals may fail to exist at all, see \cite[example 11.3]{fritz_synthetic_2020}.)

We may also want to disintegrate a joint distribution that is a function of some parameter, e.g.\ $\!\!\stringdiagram{
\wire (-2,1) -- (0,1);
\wire (0,0) -- (2,0);
\wire (0,2) -- (2,2);
\stoch{0,1}{$q$}{2.5};
\rlabel{2,0}{A};
\rlabel{2,2}{B};
\llabel{-2,1}{Z};
}$.
In this case \cref{conditional-definition-unparametrised} becomes
\begin{equation}
\label{conditional-definition}
  \stringdiagram{
 \stoch{2,0}{$q$}{4};
 \wire (-1,0) -- (2,0);
 \wire(2,1) -- (4.5,1);
 \wire(2,-1) -- (4.5,-1);
 \rlabel{4.5,1}{B};
 \rlabel{4.5,-1}{A};
 \llabel{-1,0}{Z};
 }
\,\, =
 \pbox{
  \stringdiagram{
  \wire[2.25] (-1,0) -- (-1,2.25) -- (7.5,2.25);
  \stoch{2,0}{$q$}{4};
 \wire (-2,0) -- (2,0);
 \blackdot{-1,0};
 \wire(2,1) -- (4,1);
 \wire(2,-1) -- (10,-1);
 \rlabel{4.3,1}{B};
 \blackdot{4,1};
 \wire[2] (5.5,-1)--(5.5,1)--(8,1);
 \wire (8,1.6) -- (10,1.6);
 \blackdot{5.5,-1};
 \ulabel{5,-1}{A};
 \stoch{8,1.6}{$c$}{3};
 \rlabel{10,1.6}{B};
 \rlabel{10,-1}{A};
 \llabel{-2,0}{Z};
 }}{.}
\end{equation}
Conceptually this is very similar.
We want the disintegration to hold for every parameter value $z\in Z$, and we define the conditional to be a function of $z$ as well as of $a\in A$.
In the discrete case, \cref{conditional-definition} is analogous to the identity $p(a,b\,|\, z) = p(a\,|\, z)\,p(b\,|\, a,z)$.
%

Bayes' theorem is closely related to conditional probability and can be expressed in a similar way.
Given a prior $\!\!\stringdiagram{
\wire (2,1)--(4,1);
\stoch{2,1}{$q$}{2}
\rlabel{4,1}{A};
}$
and a kernel
$\!\!\!\stringdiagram{
\wire (0,1)--(4,1);
\stoch{2,1}{$f$}{2}
\llabel{0,1}{A};
\rlabel{4,1}{B};
}$,
we can define a \emph{Bayesian inverse} of $f$ with respect to $q$, which is a kernel $\!\!\!\stringdiagram{
\wire (0,1)--(4,1);
\stoch{2,1}{$f^{\dagger}$}{2}
\llabel{0,1}{B};
\rlabel{4,1}{A};
}$
such that
\begin{equation}
\label{bayesian-inverse-unparametrised}
\pbox{
 \stringdiagram{
	\wire(1,0) -- (3,0);
	\stoch{1,0}{$q$}{3};
	\rlabel{3.8,0}{A};
	\blackdot{3,0};
	\wire[1.5] (3,0) -- (3,1.5) -- (7.5,1.5);
	\ulabel{7.5,1.5}{B};
	\stoch{5.75,1.5}{$f$}{3};
	\wire[1.5] (3,0) -- (3,-1.5) -- (7.5,-1.5);
	\ulabel{7.5,-1.5}{A};
}
\,\,=
 \stringdiagram{
	\wire(-1.5,0) -- (3,0);
	\stoch{-1.5,0}{$q$}{3};
	\ulabel{0,0}{A};
	\stoch{1.5,0}{$f$}{3};
	\rlabel{3.8,0}{B};
	\blackdot{3,0};
	\wire[1.5] (3,0) -- (3,1.5) -- (7.5,1.5);
	\stoch{5.75,-1.5}{$f^{\dagger}$}{3};
	\wire[1.5] (3,0) -- (3,-1.5) -- (7.5,-1.5);
	\ulabel{7.5,1.5}{B};
	\ulabel{7.5,-1.5}{A};
}
}{.}
\end{equation}
The Bayesian inverse $f^{\dagger}$ depends on the prior $q$ as well as on the kernel $f$.
\nvt{If we had chosen a different distribution in place of $q$, the Bayesian inverse $f^{\dagger}$ would be different.}
As with conditionals, Bayesian inverses are not necessarily unique, and for a given $f$ and $q$ there may be many kernels $f^{\dagger}$ that satisfy \cref{bayesian-inverse-unparametrised}.
(In fact, Bayesian inverses can be seen as a special case of conditionals.)

We may also consider the case where the prior takes a parameter, e.g. $\!\!\!\stringdiagram{
\wire (0,1)--(4,1);
\stoch{2,1}{$q$}{2}
\llabel{0,1}{Z};
\rlabel{4,1}{A};
}$.
In this case a Bayesian inverse in general also needs to depend on the parameter, which gives us the following more general definition:
\begin{equation}
\label{bayesian-inverse}
\pbox{
 \stringdiagram{
	\wire(-1,0) -- (3,0);
	\stoch{1,0}{$q$}{3};
	\rlabel{3.8,0}{A};
	\blackdot{3,0};
	\wire[1.5] (3,0) -- (3,1.5) -- (7.5,1.5);
	\ulabel{7.5,1.5}{B};
	\stoch{5.75,1.5}{$f$}{3};
	\wire[1.5] (3,0) -- (3,-1.5) -- (7.5,-1.5);
	\ulabel{7.5,-1.5}{A};
	\llabel{-1,0}{Z};
}
\,\,=
 \stringdiagram{
	\wire(-4,0) -- (3,0);
	\blackdot{-3,0};
	\wire[2] (-3,0) -- (-3,-2) -- (5.75,-2);
	\stoch{-1.5,0}{$q$}{3};
	\ulabel{0,0}{A};
	\stoch{1.5,0}{$f$}{3};
	\rlabel{3.8,0}{B};
	\blackdot{3,0};
	\wire[1.5] (3,0) -- (3,1.5) -- (7.5,1.5);
	\stoch{5.75,-1.5}{$f^{\dagger}$}{3};
	\wire[1.5] (3,0) -- (3,-1.5) -- (7.5,-1.5);
	\ulabel{7.5,1.5}{B};
	\ulabel{7.5,-1.5}{A};
	\llabel{-4,0}{Z};
}
}{.}
\end{equation}
The references \cite{cho_disintegration_2019,fritz_synthetic_2020,smithe_bayesian_2020} contain much more detail about Bayes' theorem in this form.

%


\section{More details about Bayesian interpretations}
\subsection{Unpacking Bayesian filtering interpretations}
\label{unpacking-filtering}

In this section we give some more intuition for \cref{bayesian-filtering-interpretation-definition} and then note some consequences of it.
The section deals mostly with the case where $S$, $Y$ and $H$ are discrete sets, meaning that we can reason in terms of probability distributions rather than measure theory.
In this case \cref{bayesian-filtering-interpretation-definition} can be written in a form that makes the relationship to Bayes' theorem more clear.
We define a notion of \emph{subjectively impossible input}, which is a value of $S$ that the reasoner believes with certainty will not occur as its next input.
(This does not imply that the input actually is impossible according to the true dynamics of the environment.)
We show that \cref{bayesian-filtering-interpretation-definition} puts no constraints on the reasoner's posterior after receiving a subjectively impossible input.

We also show that the possible interpretations of a machine only depend on which states can transition to which other states given which inputs, and not on the probabilities of such transitions.
In addition, we show that some machines admit no non-trivial interpretations at all.

In order to unpack \cref{bayesian-filtering-interpretation-definition} a little more, let us consider the case where $S$, $Y$ and $H$ are discrete.
Before starting we note that in the finite case, the definition of $\psi_S$, \cref{filtering-interpretation-S}, can be written as
\begin{equation}
 \psi_S(s\mg y) = \sum_{h\in H}\psi_{S,H'}(s,h\mg y).
\end{equation}
In this case, \cref{bayesian-filtering-consistency} can be written in symbols as
\begin{equation}
 \psi_{S,H'}(h,s\mg y) \gamma(y'\mg y,s) = \psi_S(s\mg y) \gamma(y'\mg y,s) \psi_H(h\mg y'),
\end{equation}
for all $s\in S, h,\in H, y,y'\in Y$.
We can cancel $\gamma(y'\mg y,s)$ from both sides on the assumption that it is positive, yielding
\begin{equation}
\label{filtering_intuitive_version}
\gamma(y'\mg y,s)>0 \,\implies\,   \psi_{S,H'}(h,s\mg y)  = \psi_S(s\mg y) \psi_H(h\mg y').
\end{equation}

The condition $\gamma(y'\mg y,s)>0$ means that $y'\in Y$ is a \emph{possible next state} when the machine starts in state $y\in Y$ and receives the input $s\in S$.
(There may be many possible next states in this situation because the machine may be stochastic.)

Let us then suppose that the machine starts in state $y$, receives an input $s$, and transitions to state $y'$. 
Let $h$ be an arbitrary element of $H$.
The number $\psi_{S,H'}(h,s\mg y)\in [0,1]$ can then be seen as the reasoner's prior probability that the next state is $h$ and the next input is $s$.
In more traditional notation we might write this as $P(H'=h, S=s)$, where we leave the state of the underlying machine implicit.
(Here we do not attempt to formalise this in terms of random variables, but simply treat it as a kind of notational shorthand for $\psi_{S,H'}(h,s\mg y)$.)

We may then regard $\psi_S(s\mg y)$ as the reasoner's prior probability that the next input is $s$, i.e. $P(S=s) = \sum_{h\in H} P(H=h,S=s)$.

However, since $\psi_H(h\mg y')$ is conditioned on $y'$ rather than $y$, we instead regard it as the reasoner's \emph{posterior} probability that $H'=h$.
(We refer to $H'$ rather than $H$ here because after it receives an input its previous ``next'' hidden state becomes its current hidden state.)
$\psi_H(h\mg y')$ therefore corresponds to what we might write as $P(H'=h \mid S=s)$.

With this informal shorthand notation \cref{filtering_intuitive_version} then says
\begin{equation}
\label{traditional-bayes-filtering}
 P(H'=h, S=s) = P(S=s)\, P(H'=h \mid S=s),
\end{equation}
which has the same appearance as a familiar identity from elementary probability theory.
It corresponds to a single step of Bayesian filtering, which we spell out in more detail in \cref{bayesian-filtering-proof}.

This shorthand notation gives some intuition for why \cref{bayesian-filtering-consistency} has the particular form it does, but it leaves the dependence on the state of the underlying machine implicit, and in so doing it obscures an important and subtle point.
In a more traditional context, $P(H'=h \mid S=s)$ is defined by
\begin{equation}
P(H'=h \mid S=s) = P(H'=h, S=s)/P(S=s) 
\end{equation}
and has no value when $P(S=s)=0$.
However, in our case $P(H'=h \mid S=s)$ is a shorthand for $\psi_H(h\mg y')$, which is defined even when $\psi_S(s\mg y) = 0$.

In the case where $P(S=s)>0$, \cref{bayesian-filtering-consistency} in the form of \cref{traditional-bayes-filtering} demands that $P(H'=h \mid S=s)$ is indeed equal to $P(H'=h, S=s)/P(S=s)$.
More precisely, if $\psi_S(s\mg y) > 0$ then we must have $\psi_H(h\mg y') =  \psi_{S,H'}(h,s\mg y)/\psi_S(s\mg y)$.
However, if $\psi_S(s\mg y) = 0$ then \cref{bayesian-filtering-consistency} puts no constraints on $\psi_{S,H'}(h,s\mg y)$ at all, or indeed on $\psi_H(h\mg y)$.

In the case where $S$ is a discrete set (even if $Y$ and $H$ are not discrete), we say that $s\in S$ is a \emph{subjectively impossible input} for a given state $y\in Y$ if $\psi_S(s\mg y) = 0$.
The point is that the reasoner believes, with certainty, that it will not receive the input $s$ as its next input.
The reasoning above says that in this situation, \emph{any} posterior over $H$ is acceptable, because Bayes' rule doesn't specify what the posterior should be.
We find this somewhat analogous to the fact that in logic one can deduce any proposition from a contradiction.
\Cref{bayesian-filtering-interpretation-definition} indeed permits any posterior in the case of a subjectively impossible input.
In fact, it even allows the posterior to be chosen stochastically in this case.

This is in a sense the minimal possible assumption we could make.
However, one could imagine addressing the issue in a different way by changing the framework, thus introducing a subtly different notion of interpretation than the one we have presented here.
One possibility would be to allow \emph{partial interpretations}, where $\psi_H$ becomes a partial function, meaning that not every state of the machine needs to have an interpretation at all.
This would allow the posterior to be undefined in the case of a subjectively impossible input, rather than merely arbitrarily defined.
Another possibility would be to strengthen \cref{bayesian-filtering-consistency} with additional conditions, forcing the posterior to be meaningful even after a subjectively impossible input.
We suspect that such an approach can lead to an interesting way to formalise improper priors, which are also about having meaningful posteriors in the case of `impossible' inputs, but we leave investigation of this to future work.

We note one other important consequence of the above reasoning, in the discrete case.
When we express \cref{bayesian-filtering-consistency} in the form of \cref{filtering_intuitive_version}, we see that it only depends on whether a transition from $y$ to $y'$ is possible given an input $s$, and not on the probability of such a transition.
Thus, for Bayesian filtering interpretations (and hence also for Bayesian inference interpretations), the only property of a machine that matters is which states can be reached from which other states (in a single step) under a given input.
(Strictly speaking this only makes sense in the discrete case, but we expect an analogous statement to this to hold more generally.)

Finally we note another consequence: \cref{filtering_intuitive_version} implies that if $\gamma(y'\mg y,s)>0$ for every $y,y',s$, then 
\begin{equation}
  \psi_{S,H'}(h,s\mg y)  = \psi_S(s\mg y) \psi_H(h\mg y'),
\end{equation}
for all $y,y'\in Y, s\in S, h\in H$.
Note that for any given $y\in Y$ there must exist some $s\in S$ such that $\psi_S(s\mg y)>0$.
It follows that $\psi_H(h\mg y')$ must be independent of $y'$ in this case.
In other words, if a machine is such that $\gamma(y'\mg y,s)>0$ for all $y,y',s$ then it only admits trivial interpretation maps, in which the beliefs are the same for every state.
Therefore the existence of \emph{any} non-trivial Bayesian filtering interpretation implies a fairly strong constraint on a discrete machine's dynamics, namely that some of its transition probabilities are zero.

\subsection{More on Bayesian filtering}
\label{bayesian-filtering-proof}

In this section we show that \cref{bayesian-filtering-interpretation-definition} does indeed correspond to Bayesian filtering, at least in the case of a deterministic machine.
Our proof of this is inspired by \cite[theorem 6.3]{jacobs_channel-based_2018}, which proves an analogous fact about conjugate priors.
The proof we give uses string diagram reasoning, which means that it holds even in the most general measure-theoretic context; we do not need to assume that the sets involved are discrete.

Since we restrict ourselves to only deterministic machines in this section, we will note a couple of things about deterministic machines before we talk about Bayesian filtering.

We first note that the condition for a machine $\gamma$ to be deterministic is
\begin{equation}
\raisebox{-0.5\height}{
	\includegraphics[width=0.7\textwidth]{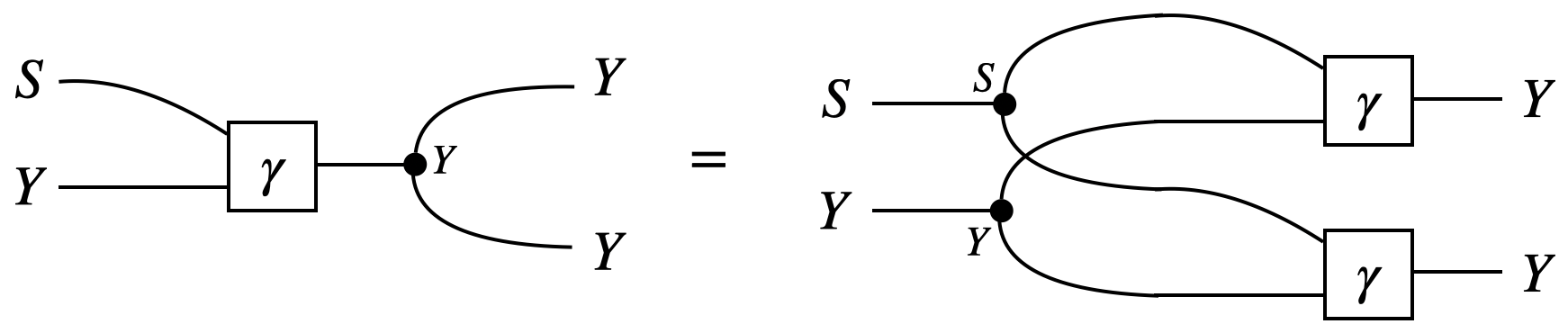}.
}
\end{equation}
This comes from the defining equation for deterministic morphisms, \cref{determinism}, and also the axiom (\ref{copy-compatibility}), noting that $\gamma$ is a kernel with input $S\otimes Y$ and output $Y$.

Next we prove the following proposition, which is useful for reasoning about Bayesian filtering interpretations of deterministic machines.
\begin{proposition}
\label{deterministic-filtering}
 Suppose $\!\!\!\stringdiagram{
\wire (0,1)--(4,1);
\wire[0.5] (0,3)--(0.5,3)--(1,2)--(2,2);
\sqbox{2,1.5}{$\gamma$}{2}
\llabel{0,3}{ S};
\llabel{0,1}{ Y};
\rlabel{4,1}{ Y};
}$ is a deterministic machine, and let $\!\!\!\stringdiagram{
\wire (0,1)--(4,1);
\stoch{2,1}{$\psi_H\!$}{2}
\llabel{0,1}{Y};
\rlabel{4,1}{H};
}$ and $\!\stringdiagram{
\wire (0,-1)--(-4,-1);
\wire[0.5] (0,-3)--(-0.5,-3)--(-1,-2)--(-2,-2);
\stoch{-2,-1.5}{$\kappa$}{2}
\rlabel{0,-3}{ S};
\rlabel{0,-1}{ H};
\llabel{-4,-1}{ H};
}$ be arbitrary Markov kernels.
Then $\psi_H$ and $\kappa$ form a consistent Bayesian filtering interpretation of $\gamma$ (i.e.\ \cref{bayesian-filtering-interpretation-definition} is satisfied) if and only if
\begin{equation}
\label{deterministic-machine-condition}
\raisebox{-0.5\height}{
	\includegraphics[width=0.7\textwidth]{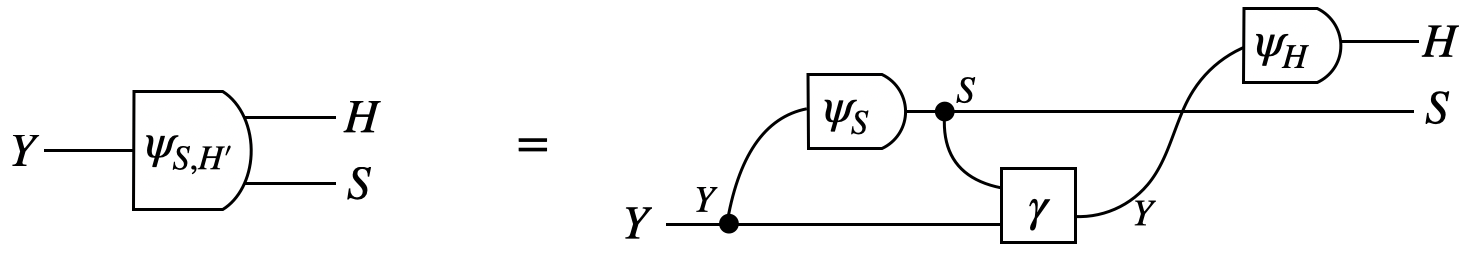},
}
\end{equation}
with $\psi_{S,H'}$ and $\psi_S$ as defined in \cref{filtering-interpretation-SH',filtering-interpretation-S}.
\end{proposition}
\begin{proof}
 To see that \cref{bayesian-filtering-interpretation-definition} implies \cref{deterministic-machine-condition} we marginalise \cref{bayesian-filtering-consistency}:
\begin{equation}
\raisebox{-0.5\height}{
	\includegraphics[width=0.8\textwidth]{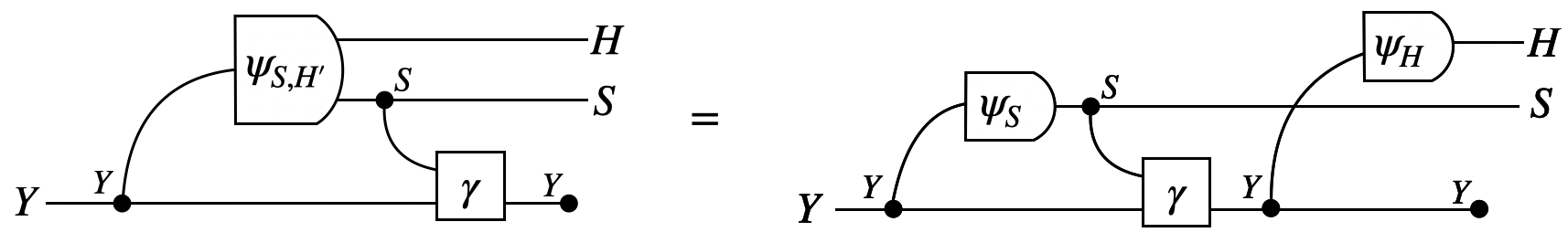}.
}
\end{equation}
This implies \cref{deterministic-machine-condition} by the rules for Markov categories, specifically \cref{naturality-of-delete,cancellation}.

For the other direction we assume \cref{deterministic-machine-condition} holds and calculate
\begin{equation}
\raisebox{-0.5\height}{
	\includegraphics[width=0.5\textwidth]{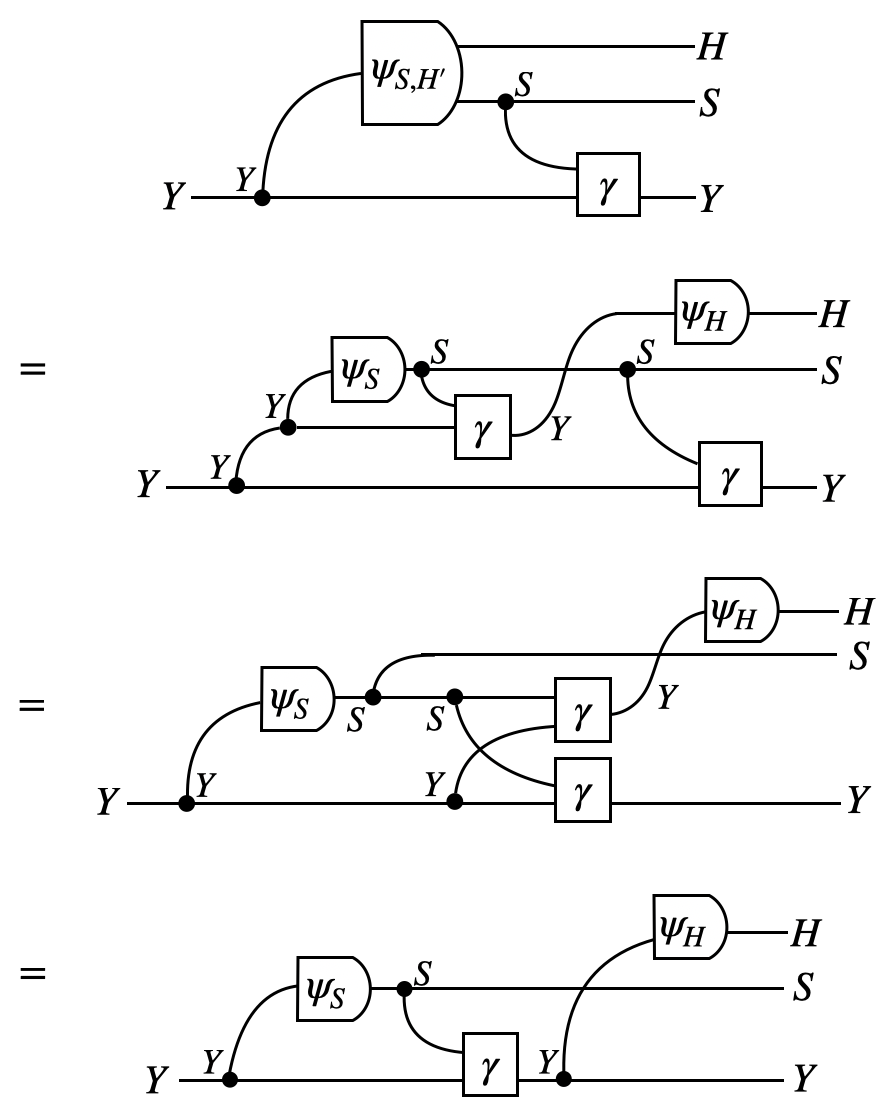}.
}
\end{equation}
The first step substitutes in the right-hand side of \cref{deterministic-machine-condition}, the second rearranges using the rules of Markov categories, and the third uses the determinism condition.
This proves that \cref{bayesian-filtering-consistency} holds.
\end{proof}

We now consider what a Bayesian filtering task involves.
The idea is that the reasoner has a model of a hidden Markov process, given by the kernel $\!\stringdiagram{
\wire (0,-1)--(-4,-1);
\wire[0.5] (0,-3)--(-0.5,-3)--(-1,-2)--(-2,-2);
\stoch{-2,-1.5}{$\kappa$}{2}
\rlabel{0,-3}{ S};
\rlabel{0,-1}{ H};
\llabel{-4,-1}{ H};
}$.
As described in the main text, this kernel can be thought of as a process that simultaneously transforms the hidden state, stochastically, into a new value and emits a visible ``sensor value.''

Given a kernel of this type, we can iterate it to produce sequences of values in $S$.
For example, we can write
\begin{equation}
\pbox{
 \stringdiagram{
\wire (0,-1)--(-4,-1);
\wire[0.5] (0,-3)--(-0.5,-3)--(-1,-2)--(-2,-2);
\stoch{-2,-1.5}{$\kappa^{3}$}{3}
\rlabel{0,-1}{H};
\rlabel{0,-3}{\,\,S^{3}};
\llabel{-4,-1}{ H};
}
\,=\,
 \stringdiagram{
\wire (0,-1)--(-10,-1);
\wire[0.5] (0,-3)--(-0.5,-3)--(-1,-2)--(-2,-2);
\stoch{-2,-1.5}{$\kappa$}{3}
\wire[1] (-0,-4) -- (-2,-4) --(-4,-2)--(-5,-2);
\stoch{-5,-1.5}{$\kappa$}{3}
\wire[1] (0,-5) -- (-4,-5) --(-7,-2)--(-8,-2);
\stoch{-8,-1.5}{$\kappa$}{3}
\rlabel{0,-3}{S};
\rlabel{0,-4}{S};
\rlabel{0,-5}{S};
\rlabel{0,-1}{H};
\llabel{-10,-1}{ H};
\ulabel{-3.5,-1}{ H};
\ulabel{-6.5,-1}{ H};
}
}{,}
\end{equation}
where $S^{3}$ means $S\otimes S\otimes S$ and $\kappa^{n}$ is notation for iterating the kernel $n$ times.
A kernel of this kind, thought of as an infinitely iterated process, is sometimes called a ``coalgebra,'' since it is a special case of a more general concept of that name. (e.g.\ \cite{jacobs_finettis_2020} takes a coalgebraic approach to de Finetti's theorem.)

For filtering we are interested in inferring the final hidden state of a system, given a finite sequence of visible states.
In order to reason about this, we define the following kernel:
\begin{equation}
\raisebox{-0.5\height}{
	\includegraphics[width=0.6\textwidth]{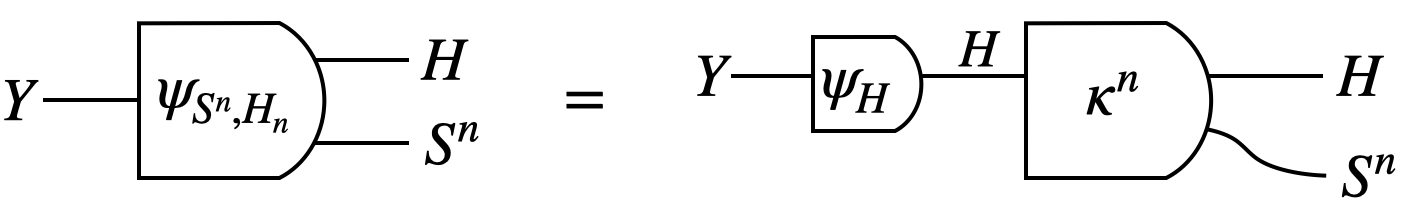}.
} 
\end{equation}
This can be seen as an interpretation map, mapping the state of a reasoner to its beliefs about its next $n$ inputs, $S^{n} = (S_1, \dots, S_n)$, along with the final value of the hidden state, $H_n$.
These take the form of a joint distribution between $S^{n}$ and $H_n$.
This joint distribution is formed from the reasoner's initial prior over the initial hidden state $H_1$ (given by the kernel $\psi_H$) and the model $\kappa$, which is iterated $n$ times.

We define this because in filtering we wish to make a probabilistic inference of the final hidden state, $H_n$, given the sequence of visible states $S^{n}$.
To infer $H_n$ given $S^{n}$ we seek a disintegration of $\psi_{S^{n},H_n}$.
(See \cref{conditional-definition-unparametrised} in \cref{conditionals-and-bayes}.)
Specifically, we seek a kernel $\psi_{H_n\mid S^{n}}\colon S^{n}\otimes Y \to P(H)$ such that
\begin{equation}
\label{filtering-disintegration}
 \raisebox{-0.5\height}{
	\includegraphics[width=0.75\textwidth]{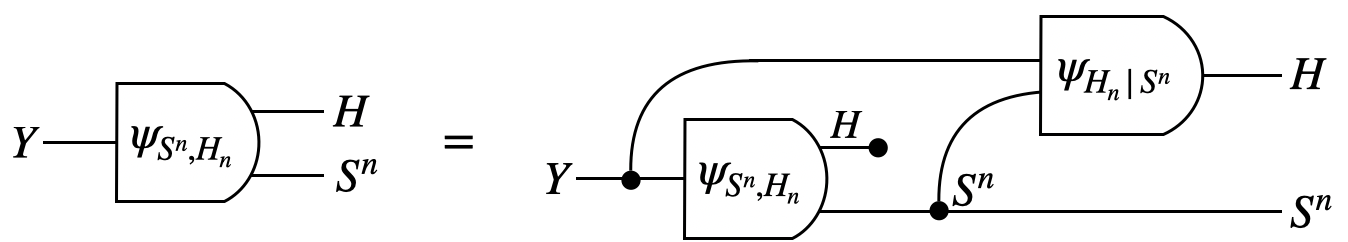}.
} 
\end{equation}
The kernel $\psi_{H_n\mid S^{n}}$ takes in a sequence $S^{n}$ of observations and returns the reasoner's conditional beliefs about $H_n$, given the sequence $S_n$.
It is also a function of the reasoner's initial beliefs $y\in Y$.

In fact such a kernel can be constructed iteratively in a natural way, if we assume that $\psi_H$ and $\kappa$ form a consistent Bayesian filtering interpretation.
To do this, we first define the iteration of $\gamma$, in a similar way to the iteration of $\kappa$:
\begin{equation}
 \raisebox{-0.5\height}{
	\includegraphics[width=0.75\textwidth]{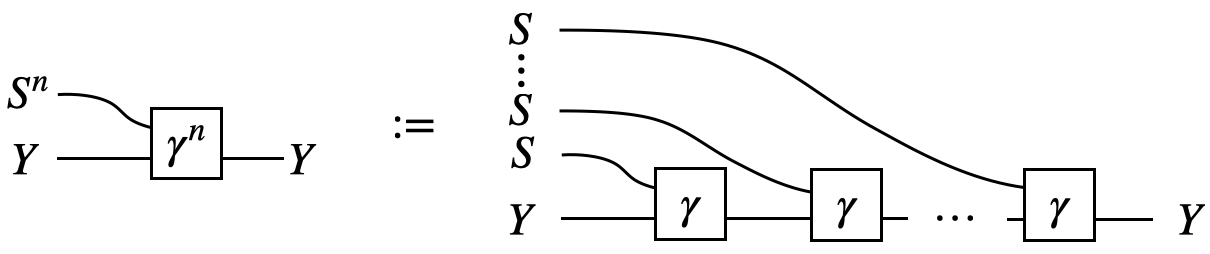},
} 
\end{equation}
where there are $n$ copies of $\gamma$ on the right-hand side. We can then state the following result, which shows that consistent Bayesian filtering interpretations can indeed be seen as performing Bayesian filtering, in the discrete case.

\begin{proposition}
 The kernel $
\stringdiagram{
\wire (-0.3,1)--(9,1);
\wire[0.5] (-0.3,3)--(0.5,3)--(1,2)--(2,2);
\sqbox{2,1.5}{$\gamma^{n}$}{2};
\llabel{-0.3,3}{S^{n}\,};
\llabel{-0.3,1}{ Y\,\,\,};
\ulabel{4,1}{ Y};
\stoch{6.2,1.5}{$\psi_{H}$}{2};
\rlabel{9,1}{H};
}
$ is a conditional of $\psi_{S^{n},H_n}$, satisfying \cref{filtering-disintegration}, in that
\begin{equation}
  \raisebox{-0.5\height}{
	\includegraphics[width=0.85\textwidth]{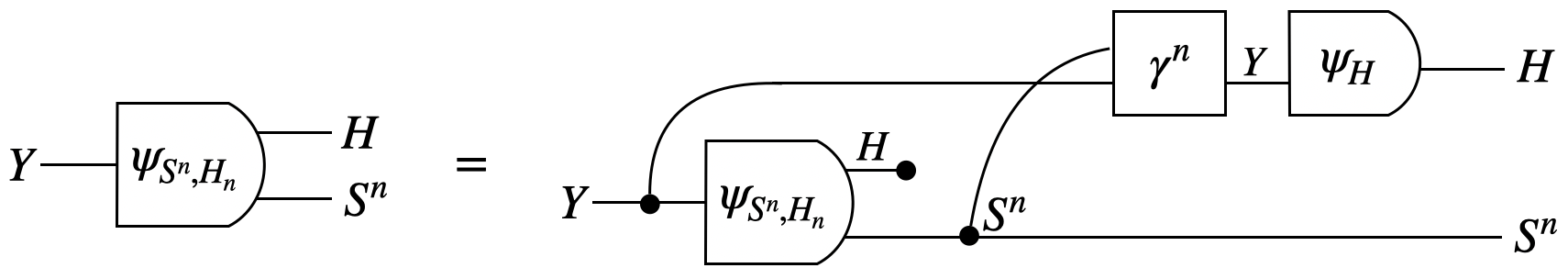}.
} 
\end{equation}
\end{proposition}
\begin{proof}
We begin by defining the kernel
\begin{equation}
  \raisebox{-0.5\height}{
	\includegraphics[width=0.6\textwidth]{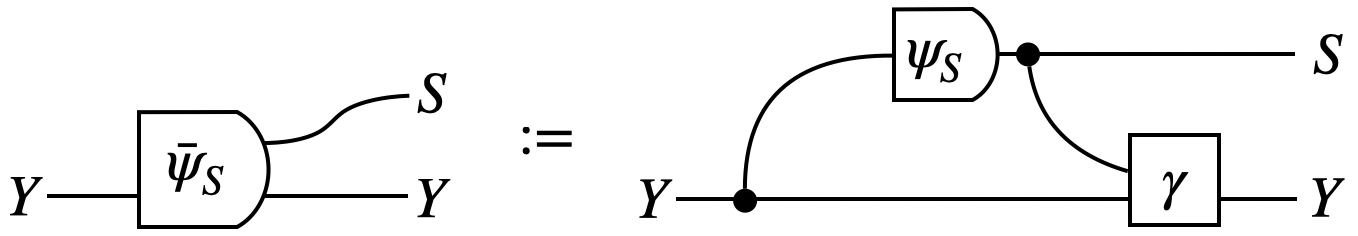}.
} 
\end{equation}
We also define its iteration, $(\bar\psi_S)^{n}\colon Y\to Y\otimes S^{n} $, analogously to $\kappa^{n}$ and $\gamma^{n}$.
We note that the consistency equation for Bayesian filtering interpretations, \cref{bayesian-filtering-consistency}, can be written in terms of $\kappa$ and $\bar\psi_S$, as
\begin{equation}
  \raisebox{-0.5\height}{
	\includegraphics[width=0.8\textwidth]{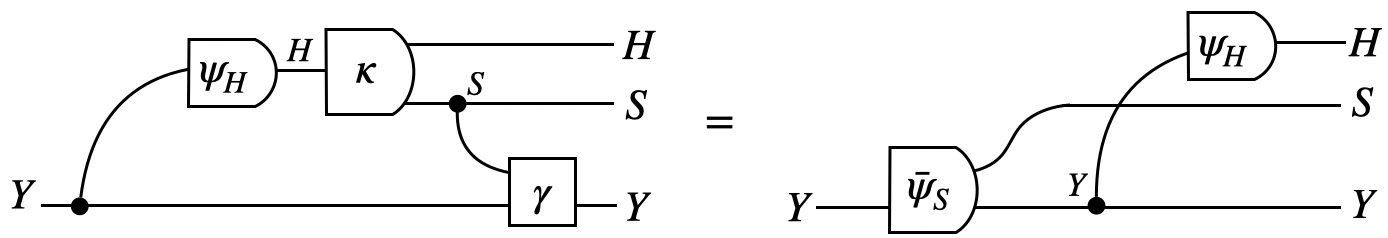}.
} 
\end{equation}
We then calculate
\begin{equation}
  \raisebox{-0.5\height}{
	\includegraphics[width=0.8\textwidth]{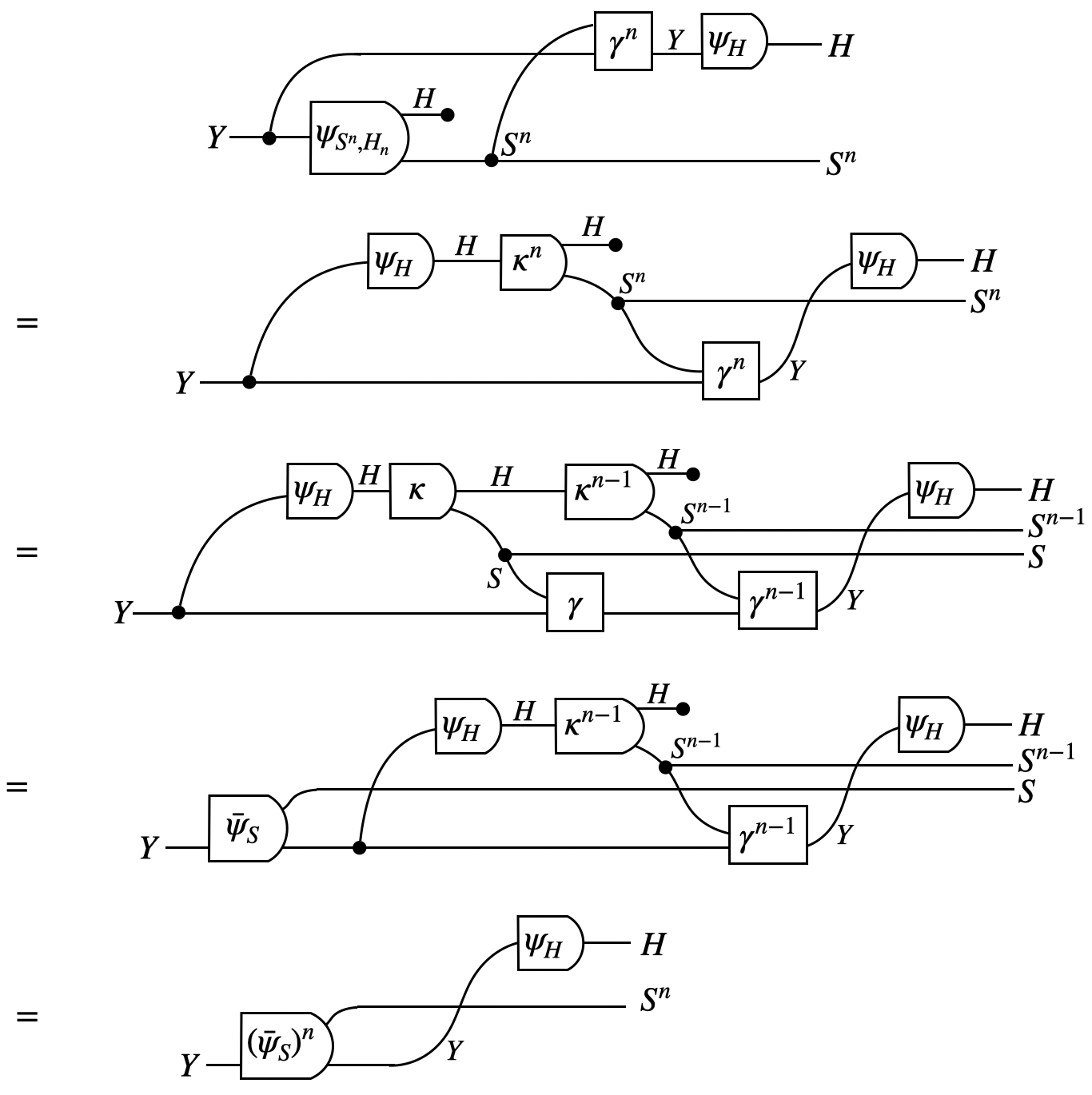},
} 
\end{equation}
where the last step is by applying the other steps inductively.
We can then apply a second inductive argument in ``the other direction'' using \cref{deterministic-machine-condition}, as follows:
\begin{equation}
  \raisebox{-0.5\height}{
	\includegraphics[width=0.7\textwidth]{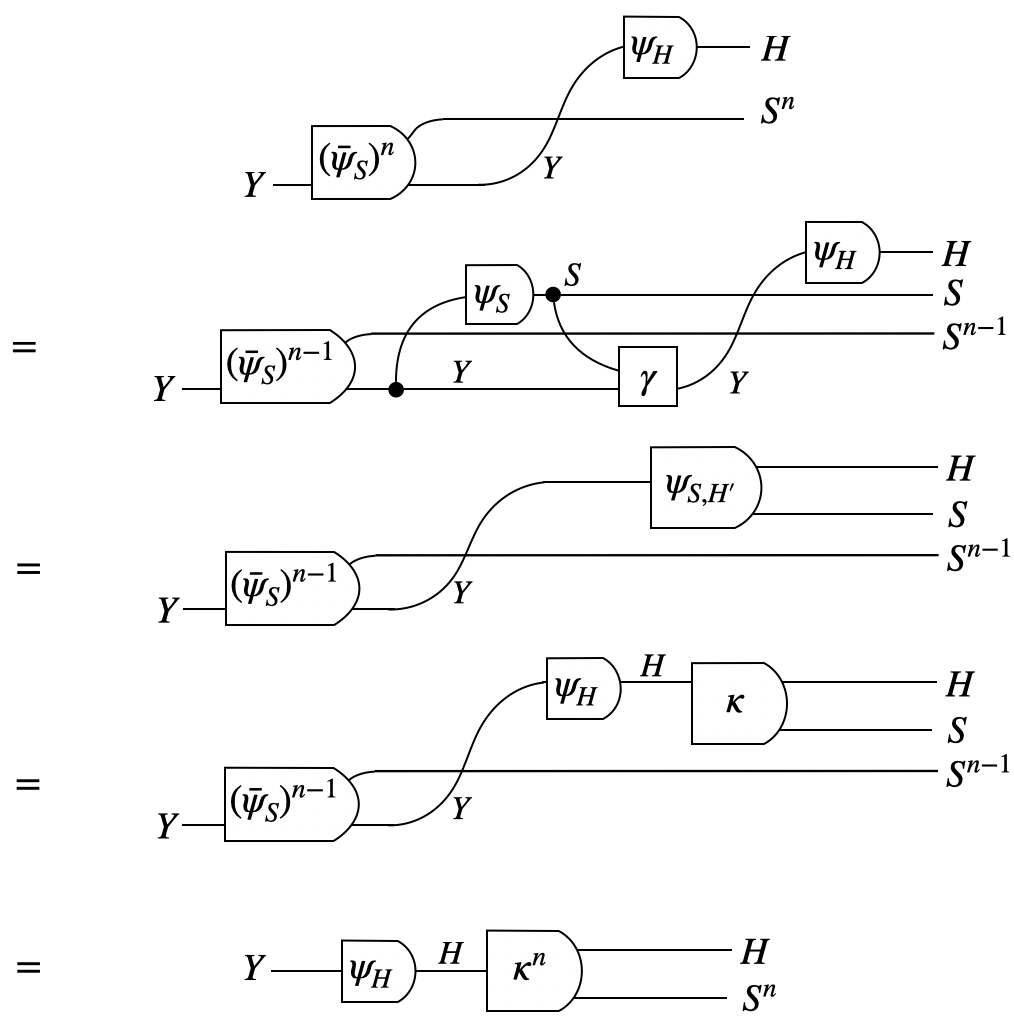},
} 
\end{equation}
where the last step is again by applying the other steps inductively.
%
\end{proof}

We have proved that
$
\stringdiagram{
\wire (-0.3,1)--(9,1);
\wire[0.5] (-0.3,3)--(0.5,3)--(1,2)--(2,2);
\sqbox{2,1.5}{$\gamma^{n}$}{2};
\llabel{-0.3,3}{S^{n}\,};
\llabel{-0.3,1}{ Y\,\,\,};
\ulabel{4,1}{ Y};
\stoch{6.2,1.5}{$\psi_{H}$}{2};
\rlabel{9,1}{H};
}
$ is a conditional of $\psi_{S^{n},H_n}$.
The kernel $
\stringdiagram{
\wire (-0.3,1)--(9,1);
\wire[0.5] (-0.3,3)--(0.5,3)--(1,2)--(2,2);
\sqbox{2,1.5}{$\gamma^{n}$}{2};
\llabel{-0.3,3}{S^{n}\,};
\llabel{-0.3,1}{ Y\,\,\,};
\ulabel{4,1}{ Y};
\stoch{6.2,1.5}{$\psi_{H}$}{2};
\rlabel{9,1}{H};
}
$
can be thought of as giving the reasoner's beliefs about $H$ after receiving a given sequence $S^{n}$ of inputs, starting from a given initial state $y\in Y$.
The result shows that these beliefs are consistent with the agent's prior $\psi_H(y)$ and the model $\kappa$, in the sense that the agent's final posterior beliefs about $H$ are a conditional of its initial joint beliefs about the sequence $S^{n}$ and the final hidden state.
We conclude that a deterministic machine with a consistent Bayesian filtering interpretation can indeed be seen as performing a Bayesian filtering task.
We expect this to be true in the general case of stochastic machines as well.

\subsection{Bayesian inference interpretations and conjugate priors}
\label{conjugate-priors}

In the main text we noted that Bayesian inference corresponds to a special case of Bayesian filtering.
By ``Bayesian inference'' here we mean the case where the reasoner is interpreted as assuming its inputs are i.i.d.\ samples from some known distribution with an unknown parameter space $H$, which we also call the hypothesis space.

The difference between inference and filtering is that we interpret the reasoner as believing that the value of $H$ is unknown but fixed.
That is, the reasoner assumes that $H$ doesn't change over time.
This corresponds to a special case of filtering in which  $\!\stringdiagram{
\wire (0,-1)--(-4,-1);
\wire[0.5] (0,-3)--(-0.5,-3)--(-1,-2)--(-2,-2);
\stoch{-2,-1.5}{$\kappa$}{2}
\rlabel{0,-3}{ S};
\rlabel{0,-1}{ H};
\llabel{-4,-1}{ H};
} = 
\!\stringdiagram{
\wire (0,0) -- (5,0);
\wire[2] (1,0) -- (1,-2) -- (5,-2);
\blackdot{1,0};
\stoch{3,-2}{$\phi$}{2};
\rlabel{5,-2}{ S};
\rlabel{5, 0}{ H};
\llabel{0, 0}{ H};
},
$
for some kernel $\phi$ that we also call the model.

While $\kappa$ can be seen as a model of the environment's dynamics, $\phi$ has more of the character of a statistical model.
It is a model of how the agent's sensor values depend on the unknown value of the hidden parameter $H$.
However, we do not put any constraints on the hypothesis space $H$ or the model $\phi$.
In particular, we do not assume that $\phi$ is an injective function $H\to P(S)$, and we allow the case where $H$ is a finite set.

In the case of inference rather than filtering, the kernels  $\psi_S$ and $\psi_{S,H'}$  from \cref{filtering-interpretation-SH',filtering-interpretation-S} can be written
\begin{equation}
\label{psi-s-inference}
 \pbox{\stringdiagram {
 \wire (0,0) -- (3,0);
 \stoch{1.5,0}{$\psi_{S}$}{3};
 \ulabel{0,0}{Y};
 \ulabel{3,0}{S};
 }
\,= 
\stringdiagram {
 \wire (0,0) -- (5.5,0);
 \stoch{1.5,0}{$\psi_H\!$}{3};
 \ulabel{0,0}{Y};
 \ulabel{2.8,0}{H};
 \stoch{4,0}{$\phi$}{3};
 \ulabel{5.5,0}{S};
 }}{}
\end{equation}
and
\begin{equation}
\label{psi-sh-inference}
 \pbox{
 \stringdiagram{
 \ulabel{0,0}{Y};
 \wire(0,0) -- (2,0);
 \stoch{2,0}{$\psi_{S,H}\!$}{5};
 \wire(2,1.25) -- (4,1.25);
 \wire(2,-1.25) -- (4,-1.25);
 \ulabel{4,1.25}{H};
 \ulabel{4,-1.25}{S};
 }
\,=
 \stringdiagram{
	\wire(0.5,0) -- (3.5,0);
	\stoch{2,0}{$\psi_H\!$}{3};
	\ulabel{0.5,0}{Y};
	\rlabel{3.8,0}{H};
	\blackdot{3.5,0};
	\wire[1.5] (3.5,0) -- (3.5,1.5) -- (7,1.5);
	\ulabel{7,1.5}{H};
	\stoch{5.5,-1.5}{$\phi$}{3};
	\wire[1.5] (3.5,0) -- (3.5,-1.5) -- (7,-1.5);
	\ulabel{7,-1.5}{S};
}}{.}
\end{equation}
%
We write $\psi_{S,H}$ instead of $\psi_{S,H'}$ because in the i.i.d.\ inference case there is only one hidden variable, that is, $H'=H$.
Thus, the joint distribution $\psi_{S,H}(y)$ can be seen as the reasoner's joint belief about its next input and the hidden variable $H$, when its underlying machine is in state $y$.
The consistency equation for Bayesian inference,
\cref{bayesian-inference-consistency}, then follows by substituting these for $\psi_{S,H'}$ and $\psi_S$ in \cref{bayesian-filtering-consistency}, the consistency equation for Bayesian filtering interpretations.

As with Bayesian filtering interpretations, it is useful to consider the case in which the underlying machine is deterministic (but not necessarily discrete).
In \cref{deterministic-filtering} we gave a simpler version of the consistency equation for Bayesian filtering interpretations, which is equivalent to \cref{bayesian-filtering-interpretation-definition} in the case of a deterministic machine.
In the inference case we can substitute \cref{psi-s-inference,psi-sh-inference} into this simplified consistency equation (\cref{deterministic-machine-condition}) to obtain
\begin{equation}
 \label{jacobs-equation}
\stringdiagram{
	\wire(0,0) -- (3.5,0);
	\stoch{2,0}{$\psi_H\!$}{3};
	\ulabel{0,0}{Y};
	\rlabel{3.8,0}{H};
	\blackdot{3.5,0};
	\wire[1.5] (3.5,0) -- (3.5,1.5) -- (7.5,1.5);
	\ulabel{7.5,1.5}{S};
	\stoch{5.75,1.5}{$\phi$}{3};
	\wire[1.5] (3.5,0) -- (3.5,-1.5) -- (7.5,-1.5);
	\ulabel{7.5,-1.5}{H};
}
=
\pbox{
\stringdiagram{
	\wire(-3,0) -- (3.5,0);
	\stoch{-0.3,0}{$\psi_H\!$}{3};
	\ulabel{0.9,0}{H};
	\stoch{2.1,0}{$\phi$}{3};
	\ulabel{-3,0}{Y};
	\blackdot{-2,0};
	\wire[1.5] (-2,0) -- (-2,-1.8) -- (5.75,-1.8);
	\rlabel{3.8,0}{S};
	\blackdot{3.5,0};
	\wire[1.5] (3.5,0) -- (3.5,1.5) -- (10,1.5);
	\ulabel{10,1.5}{S};
	\wire[1.2] (3.5,0) -- (3.5,-1.2) -- (5.75,-1.2);
	\wire(5.75,-1.5) -- (10,-1.5);
	\sqbox{5.75,-1.5}{$\gamma$}{3};
	\ulabel{7,-1.5}{Y};
	\stoch{8.2,-1.5}{$\psi_H\!$}{3};
	\ulabel{10,-1.5}{H};
}}{.}
\end{equation}
This is exactly the equation given by \cite[eq. 16]{jacobs_channel-based_2018} as a definition of a conjugate prior.

Both sides of \cref{jacobs-equation} express a joint distribution between $S$ and $H$, as a function of $Y$.
In the context of conjugate priors, $\phi$ is considered to be a family of distributions, with parameters $H$.
Our interpretation map $\psi_H$ corresponds to another family of distributions, which is a conjugate prior to $\phi$.
The machine state $Y$ corresponds to the so-called hyperparameters, i.e.\ the parameters of $\psi_H$.

This shift in perspective makes sense.
In a computational context, conjugate priors are often useful precisely because they offer a way to perform inference without needing to directly calculate Bayesian inverses at run-time.
Instead, the implementation only needs to keep track of the hyperparameters and update them in response to data.
This update takes place according to a deterministic function, whose form depends on the family $\phi$ and its conjugate prior $\psi_H$.
This updating of the hyperparameters is the role played by $\gamma$: it takes in a data point in $S$ along with the current value of the hyperparameters, and returns the updated hyperparameters.
\Cref{jacobs-equation} asserts that this must done in such a way that the new value of $Y$ does indeed correspond to the correct Bayesian posterior, when mapped to a distribution over $H$ by the kernel $\psi_H$.

We note that it is somewhat nontrivial to find a pair of kernels $\psi_H$, $\phi$ and a function $\gamma$ such that \cref{jacobs-equation} is obeyed.
However, many such examples are known.
(Although it is not an authoritative source, a useful list can be found online \cite[under ``Table of conjugate distributions'']{wikipedia_conjugate_prior}, which explicitly gives both kernels and the update function for each example.) 
Any example of a conjugate prior can be seen as a deterministic machine together with a consistent Bayesian inference interpretation.
In addition, in \cref{app:examples} we give a number of examples of a different flavour, in that in our examples $H$ is either a finite or a countable set.

\subsection{Unpacking Bayesian inference interpretations}
\label{unpacking-inference}

We now unpack \cref{bayesian-inference-interpretation-definition} by converting \cref{bayesian-inference-consistency} into more familiar terms in the case where all the spaces are discrete sets, as we did for filtering interpretations in \cref{unpacking-filtering}.

In the case where $Y$, $H$ and $S$ are finite sets, \cref{bayesian-inference-consistency} can be written as
\begin{equation}
\label{bayesian-consistency-kernels}
 \psi_H(h\mg y)\, \phi(s\mg h)\, \gamma(y'\mg s,m) = \psi_S(s\mg y)\, \gamma(y'\mg s,y)\, \psi_H(h\mg y'), \\
\end{equation}
or equivalently,
\begin{equation}
\label{bayes-in-kernels}
	\gamma(y'\mg s,y)>0\quad\implies\quad \psi_H(h\mg y)\, \phi(s\mg h) = \psi_S(s\mg y)\, \psi_H(h\mg y'),
\end{equation}
since we can cancel $\gamma(y'\mg s,y)$ if we assume it is positive.
For $\gamma(y'\mg s,y)$ to be positive means that it is possible for the machine to transition from state $y\in Y$ to state $y'\in Y$ after receiving the input $s\in S$.

We can now give an intuitive interpretation to the terms in this equation.
If the machine starts in state $y$, receives input $s$, and transitions to state $y'$ as a result, then we can regard $\psi_H(h\mg y)$ as the reasoner's prior beliefs about the hypothesis $h$,  $\psi_S(s\mg y)$ as its prior beliefs about the input $s$, and $\psi_H(h\mg y')$ as the reasoner's posterior belief about the hypothesis $h$. \Cref{bayes-in-kernels} can then be compared, term by term, to the much more familiar equation
\begin{equation}
 p(h)\,p(s\mid h) = p(s)\, p(h\mid s).
\end{equation}
Here we have written $p(s\mid h)$ in place of $\phi(s\mg h)$ and $p(h\mid s)$ in place of $\psi_H(h\mg y')$ in order to emphasise the similarity to Bayes' theorem in a more familiar form.
Our definition, in the form of \cref{bayesian-inference-consistency} or \cref{bayesian-consistency-kernels}, differs from this in that it explicitly takes account of the machine's state, and $\phi$ and $\psi_H$ are defined by Markov kernels rather than conditional probabilities.

We note that, as in the case of filtering (\cref{unpacking-filtering}), our definition of a consistent Bayesian inference interpretation allows the posterior to be arbitrary in the case of subjectively impossible inputs, i.e.\ those $s\in S$ for which $\sum_{h\in H}\psi_H(h\mg y)\phi(s\mg h) = 0$ for a given state $y\in Y$.
Given such an input the reasoner may update its posterior to anything at all.
As with filtering, we regard this as the minimal assumption we could have made, but we can imagine several other choices that one could make instead.
These include allowing the posterior to be undefined in such cases; \emph{requiring} it to be undefined; requiring it to obey some additional consistency equation such that the posterior would make sense even on subjectively impossible inputs; or requiring $\phi$ and $\psi_H$ to be such that subjectively impossible inputs do not exist.
We would consider these to be subtly different kinds of interpretation, and we leave their further investigation to future work.

\section{Details of examples}
\label{app:examples}

\subsection{An Interpretation Of A Non-Deterministic Finite Machine}
\label{app:three_state_machine}
We here present a non-deterministic finite machine with internal state space $Y = \{ y_0, y_1, y_2 \}$ and sensory input space $S = \{ s_1, s_2 \}$. 
\begin{figure}
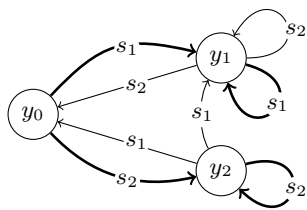

\ctikzfig{3-node-machine} 
\caption{Transitions for a three-state machine. Deterministic transitions are depicted using bold arrows, and non-deterministic transitions using regular arrows. The precise probability values for non-deterministic transitions are not shown, since we only need to know that they are non-zero.}
\end{figure}
One Bayesian interpretation of this machine, for a hidden state space $H = \{ h_1, h_2 \}$, is as follows (where $\delta$ is the Kronecker delta):
\begin{nalign}
	\phi(s_i \mg h_j) &= \delta_{ij} \\
	\psi(h_i \mg y_j) &= \begin{cases}
		\delta_{ij} \text{ if } j \in \{ 1, 2 \} \\
		0.5 \text{ if } j = 0. \\
	\end{cases}
\end{nalign}
Under this interpretation, the model $\phi$ ascribed to the machine is that sensory inputs transparently reflect the hidden state. 
The machine, in internal state $y_0$, is taken to be uncertain about the hidden state; in state $y_i \in \{ y_1, y_2 \}$ it is taken to be certain that the hidden state is $h_i$. 
The dynamics of the machine match this interpretation: it transitions deterministically to $y_i$ when receiving input $s_i$, unless $s_i$ is ``subjectively impossible'' ($s_2$ at $m_1$, and $s_1$ at $m_2$). 
Behaviour on subjectively impossible inputs is not constrained by the consistency equation, so this is a consistent Bayesian interpretation.

\subsection{Machine counting occurrences of different observations}
\label{app:counting_machine}
We now consider a countably infinite deterministic machine $(Y_0,S,\gamma_0)$. 
Let $Y_0=\mathbb{N}^+\times \mathbb{N}^+$ ($\mathbb{N}^+$ excludes $0$) and the input space be $S = \{ +1, -1 \}$. 
The machine deterministically computes the function $f_0: (Y_0 \times S) \to Y_0$, in the sense that $\gamma_0(f_0(y, s) \mg y, s) = 1$. 
Essentially, it keeps distinct count of how many $+1$ and $-1$ inputs it has received. Formally:
\begin{nalign}
f_0((i, j), s) &= \begin{cases}
	(i+1, j) \text{ if } s = +1 \\
	(i, j+1) \text{ if } s = -1
\end{cases}
%
%
\end{nalign}
One consistent Bayesian interpretation $(\psi_0,\phi_0)$ for machine $\gamma_0$ uses hypothesis space $H_0 = [0, 1]$ and model:
\begin{nalign}
\phi_0(s \mg h) = h^{\delta_{-1}(s)}(1-h)^{\delta_{+1}(s)}
\end{nalign}
where
\begin{align}
  \delta_{u}(v):=\begin{cases}
	1 \text{ if } u = v \\
        0 \text{ else.}
\end{cases}
\end{align}
This model is known as the categorical distribution for two outcomes (or just the Bernoulli distribution).
The machine states were deliberately chosen to be the hyperparameters of a possible interpretation map $\psi_0:Y_0 \to H_0$ which is known as the Dirichlet distribution (and in this special case also as Beta-distribution):
\begin{nalign}
\psi_0(h \mg (i,j)) = \frac{1}{B(i,j)} h^{i-1} (1-h)^{j-1}
\end{nalign}
where $B(i,j)$ is the Beta function.

This interpretation map (the Dirichlet distribution) is the conjugate priors for categorical distributions. 
This implies that $(\psi_0,\phi)$ form a consistent Bayesian inference interpretation, as explained in \cref{conjugate-priors}.

%
%

\subsection{Machine tracking differences between the number of occurrences of different observations}
We now consider another countably infinite deterministic machine $(Y_1,S,\gamma_1)$ which has the same input space as the machine in \cref{app:counting_machine}. 
Let $Y_1=\mathbb{Z}$ and the input space again be $S = \{ +1, -1 \}$. 
The machine $\gamma_0$ deterministically computes a function $f_1: (Y_1 \times S) \to Y_1$, in the sense that $\gamma_1(f_1(y, s) \mg y, s) = 1$. 
Machine $\gamma_1$ only counts how many more $+1$ inputs it has received than $-1$ inputs. Formally:
\begin{nalign}
f_1(k, s) &= k+s.
\end{nalign}
One consistent Bayesian interpretation $(\psi_1,\phi_1)$ for machine $\gamma_1$ uses hypothesis space $H_1 = \{ h_{+1}, h_{-1} \}$ and model:
\begin{nalign}
\phi(s \mg h_i) = \begin{cases}
0.75 \text{ if } i = s \\
0.25 \text{ otherwise. }
\end{cases}
\end{nalign}
The interpretation map $\psi_1:Y_1 \to H_1$ is
\begin{nalign}
\psi_1(h_{+1} \mg k) = \frac{1}{2(1 + 0.75^{k} 0.25^{-k})} \\
\psi_1(h_{-1} \mg k) = \frac{1}{2(1 + 0.75^{-k} 0.25^{k})}
\end{nalign}
It is relatively easy to verify that $(\psi_1,\phi_1)$ is a consistent Bayesian interpretation with $\gamma_1$'s dynamics. 

As a teaser for future work we may note the following.
Since machine $\gamma_0$ of \cref{app:counting_machine} stores the individual counts for $s_0$ and $s_1$ inputs, it also implicitly keeps track of the difference between those counts; $\gamma_1$ only keeps track of this difference. 
Consequently, we can define a deterministic kernel $g:Y_0 \to P(Y_1)$ that maps any state $(i, j) \in Y_0$ of $\gamma_0$ to $g(i,j) \coloneqq \delta_{i-j}$ which is a probability measure over the state space of $\gamma_1$. 
It turns out that for this map, for any $k' \in Y_1,s \in S$ and $(i,j) \in Y_0$ we have
\begin{align}
\label{eq:example_machine_map}
  (\gamma_0\comp g)(k' \mg (i,j),s)=\sum_k \gamma_1(k' \mg k,s) g(k \mg(i,j)).
\end{align}
This implies that we can construct an interpretation of machine $\gamma_0$ from the interpretation $(\psi_1,\phi_1)$ of $\gamma_1$. 
For this we precompose the interpretation map $\psi_1$ for $\gamma_1$ with the machine map $g$ to get a consistent Bayesian inference interpretation $(g \comp \psi_1,\phi_1)$ for $\gamma_0$.
In future work we intend to further develop the theory of how a consistent interpretation of one deterministic machine can be ``pulled back'' to other machines that are related in a similar way to \cref{eq:example_machine_map}. 

\section{Details on the relation to the FEP}
\label{app:FEP}

We here try to identify the structures in the FEP that are analogous to the notions of machine $\gamma$, model $\kappa$, and interpretation map $\psi_H$.  
This suggests that, at least in some treatments of FEP, there is an implicit concept that is close to what we have called a reasoner.
We will call this putative concept the FEP reasoner.
%

Large parts of the FEP literature do not explicitly deal with FEP reasoners but are sometimes presented as based on them (e.g.\ in \cite{friston_sophisticated_2021}).
The parts that construct the FEP reasoner are those called ``Bayesian mechanics'' and are still evolving.
A standard reference is \cite{friston_free_2019} but this is known to contain some issues \cite{biehl_technical_2021,friston_interesting_2020,aguilera_how_2021}.
The most recent version can be found in \cite{da_costa_bayesian_2021}.

Understanding precisely the relationship between the concepts of our Bayesian and the FEP reasoner is future work.
The following are preliminary observations.

\subsection{Machine}
We first identify the structure in the FEP setup that is most closely related to a machine and is also said to appear to perform Bayesian inference.
Unfortunately, the latest iteration of the conditions under which there exists an FEP reasoner, which is \cite{da_costa_bayesian_2021}, does not make this particular structure as explicit as the previous version \cite{friston_free_2019}.
We will therefore identify this structure in the older version.
A corresponding structure should also exist in the newer version and we will hint at how it may differ.

The FEP setup in \cite{friston_free_2019} consists of four sets of variables $\eta \in E,s \in S,a \in A,\mu \in M$ called external, sensory, active, and internal states with $E,S,A,M$ finite dimensional real vector spaces. 
These variables obey the stochastic differential equations  
\begin{nalign}
\label{eq:componentwise}
\dot{\eta}&=f_\eta(\eta,s,a) + \omega_\eta\\
    \dot{s}&=f_s(\eta,s,a) + \omega_s\\
      \dot{a}&=f_a(s,a,\mu) + \omega_a\\
      \dot{\mu}&=f_\mu(s,a,\mu) + \omega_\mu
\end{nalign} 
where $\omega_\eta,\omega_s,\omega_a,\omega_\mu$ are independent Gaussian noise terms. 
The FEP goes beyond the scope of a reasoner and formulates a concept of agent.
The concept of an agent should, as part of its interpretation, make it possible to talk about deliberate actions.
In the FEP deliberate actions are associated to the active states $a$.
At the same time, the internal states are only involved in inference (or filtering) and the special case where there are no active states seems to be within the scope of the FEP.
This should still leave us with a FEP reasoner and make it more comparable to our Bayesian reasoner.
We therefore consider the special case where there are no active states such that we get:
\begin{nalign}
\label{eq:componentwise2}
\dot{\eta}&=f_\eta(\eta,s) + \omega_\eta\\
    \dot{s}&=f_s(\eta,s) + \omega_s\\
      \dot{\mu}&=f_\mu(s,\mu) + \omega_\mu.
\end{nalign} 
This looks like a continuous time version of the Bayesian network in \cref{bayesian-network-context} and has the somewhat significant feature that all influences from the external states $\eta$ are mediated by the sensory states $s$. 
This suggests that it is possible to see the sensory states $s \in S$ as inputs to a machine state $\mu \in M$ with the external states $\eta \in E$ ``hidden behind'' the sensory states.   

The internal states $\mu \in M$ are supposed to appear to infer the external states.
So the state space $Y$ of the machine of the FEP reasoner should be identified with $M$. 
Going by their name and their role in the earlier dynamics of \cref{eq:componentwise} it seems reasonable to identify the sensory state space $S$ with the input state space (also $S$ in our notation) of the machine.

This brings us to the machine's kernel $\gamma$. 
Our formalism does not deal with continuous-time kernels at the moment so we only make some informal comments here.
Note that none of the following statements should be considered as proven.
Since all variables together form a (time-homogeneous) Markov process, we can choose times $t,t+\tau$ with $\tau>0$ and write the conditional probability density (assuming things are well behaved enough) at a state $(\eta',s',\mu')$ at $t+\tau$ given a state $(\eta,s,\mu)$ at time $t$ as $p(\eta',s',\mu',t+\tau\mid \eta,s,\mu,t)$ (this notation is taken from \cite[p.31]{risken_fokker-planck_1996}).
We can then marginalise out $\eta'$ and $s'$ to get $p(\mu',t+\tau\mid \eta,s,\mu,t)$ which looks a bit closer to a machine kernel but still depends also on $\eta$. 
We cannot just drop $\eta$ from this expression even if we assume \cref{eq:componentwise2} holds since within a time interval $[t,t+\tau]$ with $\tau >0$ the influence from $\eta$ would propagate through the intermediate values of the sensory states to $\mu'$. 
Instead we here condition on all those intermediate values of the sensory state between $t$ and $t+\tau$.
Write $s[t:t+\tau]$ for a part of the trajectory of the sensory state between $t$ and $t+\tau$ that starts in $s$.
Then, assuming \cref{eq:componentwise2} we should get:
\begin{align}
p(\mu',t+\tau \mid \eta,s[t:t+\tau],\mu,t)=p(\mu',t+\tau\mid s[t:t+\tau],\mu,t).
\end{align}
In order to make this look even more like a kernel we may take the limit as $\tau \to 0$ and so we write 
\begin{align}
  \gamma(\mu'\mid  \mu,s):= \lim_{\tau \to 0} p(\mu',t+\tau\mid s[t:t+\tau],\mu,t)
\end{align}
which is just a notation for an expression that hopefully provides sufficient intuition for our purposes.

%
What is important is that within the system \cref{eq:componentwise2} there should be a (continuous-time) machine describing the dynamics of the internal states in response to sensory states.

In \cite{da_costa_bayesian_2021} the structure of \cref{eq:componentwise} and thus \cref{eq:componentwise2} is no longer assumed.
This means all variables can directly influence each other and the sensory states no longer play a special role.
However, the sensory (and usually the active states) are still special due to an additional assumption.
The larger process has to have a stationary distribution $p(\eta,s,a,\mu)$ that factorises according to 
\begin{align}
\label{eq:mblanket}
  p(\eta,s,a,\mu)=p(\eta|s,a)p(\mu|s,a)p(s,a)
\end{align}
which is referred to as a Markov blanket.
%
%
%
With this assumption \mbcol{only,} one can no longer assume that the sensory states $s[t:t+\tau]$ can ``shield'' those states from direct influence by external states, which makes it more difficult to compare the dynamics to our setup.
%
%
%
A solution may be to use 
a continuous-time version of the approach in \cite{rosas_causal_2020}.
Below we ignore this issue and assume that we have the structure of \cref{eq:componentwise2}. 
\subsection{Model}
For a reasoner we also need a model and an interpretation map.
%
%
%
%
As already mentioned the FEP assumes that the system in \cref{eq:componentwise} has a stationary distribution $p(\eta,s,\mu)$.
%
One purpose of this assumption seems to be the definition of what we call the model.
%
\mbcol{In the language of the FEP literature the stationary distribution defines the generative model. 
Here, generative model refers to a joint probability distribution over causes (parameters/hidden variables) and observed variables.
In \cite[Section 3.b)]{da_costa_bayesian_2021} the generative model is defined to be $p(\eta,s,\mu)$ with $\eta$ as the hidden variables and observed variables $(s,\mu)$.}
We are not sure why $\mu$ is also seen as an observed variable and not only $s$.
This could mean that the machine state $\mu$ itself is also modelled by an FEP reasoner.
However, this would need further investigation that we leave for future work.
So we resort to a previous version where only the marginalised stationary distribution $p(\eta,s)$ was considered as the generative model (\cite[Fig.3]{parr_markov_2020},\cite[p.101]{friston_free_2019}).
In that case the hidden variable space $H$ in our notation should be identified with the external state space $E$ and the model (in our sense) is a conditional distribution induced by the stationary distribution:
\begin{align}
  \phi(s\mid \eta) \coloneqq p(s\mid\eta).
\end{align}
Note that, this choice of a model by itself does not immediately tell us whether the FEP reasoner does filtering or just inference in the sense of \cref{bayesian-inference-interpretation-definition}.
A model like $\phi(s\mg \eta)$ can be part of a filtering kernel $\kappa$ as well. 
In both cases we also need an interpretation map.

\subsection{Interpretation map}
For the interpretation map $\psi_H$ we need a kernel of type $M \to P(E)$.
Indeed, a kernel that has the right type can be identified in the FEP literature.
This kernel is denoted $q_\mu(\eta)$ and we will identify $\psi_H(\eta \mg \mu)=q_\mu(\eta)$.
The kernel's definition, however, relies on another assumption of the FEP, namely the existence of a ``synchronisation map'' $\sigma:M \to E$.
To construct $\sigma$ let us first define two other functions $g_M:S \to M$ and $g_E:S \to E$ via
\begin{nalign}
  f_M(s) &\coloneqq  \mathbb{E}_{p(\mu|s)}[\mu]\\
  f_E(s) &\coloneqq \mathbb{E}_{p(\eta|s)}[\eta]
\end{nalign}
and then set 
\begin{align}
  \sigma(\mu) \coloneqq f_E(f_M^{-1}(\mu))
\end{align}
which is assumed to be well defined.
For details on when this exists in the linear case see \cite{aguilera_how_2021,da_costa_bayesian_2021}.
With this we can define $q_\mu(\eta)$ and in turn the interpretation map $\psi_H$. 
This maps an internal state $\mu$ to the Gaussian distribution with mean value equal to $\sigma(\mu)$:
\begin{align}
\label{eq:FEPinterpretationmap}
  \psi(\eta \mg \mu) \coloneqq q_{\mu}(\eta) \coloneqq \mathcal{N}(\eta; \sigma(\mu),\Sigma(\mu))
\end{align}
where the variance $\Sigma(\mu)$ is defined as the variance of the best Gaussian approximation to the model $p(s|\eta=\sigma(\mu))$ when the external state is equal to $\sigma(\mu)$ \cite[Eq.2.4]{parr_markov_2020}.
Note that in \cite{da_costa_bayesian_2021} the whole stationary distribution is assumed as Gaussian and so $p(\eta|\mu)$ in the corresponding equation in that publication (i.e.\ Eq.3.3) is also a Gaussian.

In conclusion, the necessary ingredients for something like a Bayesian reasoner seem to be present in the FEP literature.
One thing that is special about the FEP reasoner is that its model $\kappa$ and interpretation map $\psi_H$ are derived from features of the process that the machine is embedded in. 

%
%
We do not know whether there is an appropriate notion of consistency equation that the FEP reasoner obeys.
Presumably, instead of the equation for exact inference that we have presented, such an equation would express the idea that the FEP reasoner performs approximate inference in the form of free energy minimisation.
Other differences are that the FEP takes place in continuous time, and perhaps more significantly, that it deals with deliberate actions as well as inference.
However, it is not inconceivable that these could be expressed in the form of a consistency equation.
%
%
%

%
In the current formulations of the FEP, the interpretation is derived from the properties of the `true' environment, such as the stationary distribution, or the synchronisation map $\sigma$.
In our consistency equation approach, this need not be the case, since a reasoner's beliefs only need to be consistent and need not be correct.
This means in particular that no stationarity assumption is needed.

Nonetheless, perhaps an important idea behind the FEP is that the model that most closely corresponds to the true environment can be considered the best one.
%
%
A consistency equation approach would still be helpful, in order to systematically explore whether and how interpretations should relate to the larger process in which the machine is embedded.

%
%
%
%
%
%
%


\end{document}